\documentclass{article}

\newif\ifICML
\ICMLtrue %

\usepackage[accepted]{icml2024}

\usepackage{amsmath,amssymb,pifont}
\usepackage{amstext}
\usepackage{amsthm}
\usepackage{booktabs}
\usepackage[skip=0pt]{subcaption}
\usepackage{times}
\usepackage{lipsum}
\usepackage[shortlabels]{enumitem}
\usepackage{xspace}
\usepackage{cancel}
\usepackage{wrapfig}
\usepackage{array}
\usepackage{siunitx}
\usepackage{csvsimple}
\usepackage[multidot]{grffile}
\usepackage{bbm}
\usepackage{dblfloatfix}
\usepackage{hyperref}
\usepackage{makecell}
\usepackage{bbm, dsfont}
\usepackage{mathtools}
\mathtoolsset{showonlyrefs}
\usepackage{xcolor}
\usepackage{comment}
\usepackage{graphicx}
\usepackage[multiple]{footmisc}
\usepackage{mathrsfs}
\usepackage{todonotes}
\usepackage{tikz}

\usepackage{algorithm,algorithmic}

\usepackage{eqparbox}

\usepackage{cleveref}

\renewcommand{\todo}[1]{}

\hypersetup{final}

\newcommand{\cov}{\Sigma}
\newcommand{\noise}{\bfz}
\newcommand{\bignoise}{\mathbb{Z}}
\newcommand{\multiplier}{c_0}

\newcommand{\eps}{\ensuremath{\varepsilon}}

\newcommand{\noclippingcondition}{No-Clipping Condition\xspace}

\newcommand{\bfA}{\ensuremath{\mathbf{A}}}

\newcommand{\bfD}{\ensuremath{\mathbf{D}}}

\newcommand{\bfT}{\ensuremath{\mathbf{T}}}

\newcommand{\bfX}{\ensuremath{\mathbf{X}}}

\newcommand{\bfg}{\ensuremath{\mathbf{g}}}

\newcommand{\bfu}{\ensuremath{\mathbf{u}}}
\newcommand{\bfv}{\ensuremath{\mathbf{v}}}

\newcommand{\bfx}{\ensuremath{\mathbf{x}}}
\newcommand{\bfy}{\ensuremath{\mathbf{y}}}
\newcommand{\bfz}{\ensuremath{\mathbf{z}}}

\newcommand{\calA}{\ensuremath{\mathcal{A}}}

\newcommand{\calM}{\ensuremath{\mathcal{M}}}
\newcommand{\calN}{\ensuremath{\mathcal{N}}}
\newcommand{\calO}{\ensuremath{\mathcal{O}}}

\newcommand{\calX}{\ensuremath{\mathcal{X}}}
\newcommand{\calY}{\ensuremath{\mathcal{Y}}}

\newcommand{\poly}[1]{\ensuremath{{\rm poly}\left(#1\right)}}

\renewcommand{\Pr}{\mathop{\mathbf{Pr}}}

\newcommand{\E}{\mathop{\mathbf{E}}}

\newcommand{\R}{\mathbb{R}}
\newcommand{\I}{\mathbb{I}}

\newtheorem{lem}{Lemma}[section]

\newtheorem{thm}[lem]{Theorem}

\newtheorem{defn}[lem]{Definition}
\newtheorem{fact}[lem]{Fact}

\newtheorem{condition}[lem]{Condition}
\newtheorem{claim}[lem]{Claim}

\newcommand{\bracket}[1]{\left(#1\right)}
\newcommand{\sqbracket}[1]{\left[#1\right]}

\makeatletter
\newcommand{\vast}{\bBigg@{4}}
\newcommand{\Vast}{\bBigg@{5}}
\makeatother
\DeclareMathOperator{\tr}{\mathrm{tr}}

\newcommand{\ip}[2]{\langle #1, #2\rangle}

\newcommand{\norm}[1]{\| #1 \|}

\DeclarePairedDelimiter\abs{\lvert}{\rvert}

\DeclarePairedDelimiterX{\infdivx}[2]{(}{)}{%
  #1\;\delimsize\|\;#2%
}

\renewcommand{\epsilon}{\varepsilon}

\renewcommand{\tilde}{\widetilde}

\begin{document}

\twocolumn[
\icmltitlerunning{Private Gradient Descent for Linear Regression}

\icmltitle{Private Gradient Descent for Linear Regression: Tighter Error Bounds and Instance-Specific Uncertainty Estimation}

\icmlsetsymbol{equal}{*}

\begin{icmlauthorlist}
\icmlauthor{Gavin Brown}{uw}
\icmlauthor{Krishnamurthy (Dj) Dvijotham}{deepmind}
\icmlauthor{Georgina Evans}{deepmind}
\icmlauthor{Daogao Liu}{uw}
\icmlauthor{Adam Smith}{bu,deepmind}
\icmlauthor{Abhradeep Thakurta}{deepmind}
\end{icmlauthorlist}

\icmlaffiliation{uw}{Paul G Allen School of Computer Science and Engineering, University of Washington. While at UW, G.B. was supported by NSF Award 2019844. Part of this work was done while G.B. was at Boston University.}
\icmlaffiliation{deepmind}{Google DeepMind}
\icmlaffiliation{bu}{Department of Computer Science, Boston University. A.S. and G.B., while at BU, were supported in part by NSF awards CCF-1763786 and CNS-2120667 as well as Faculty Awards from Google and Apple}

\icmlcorrespondingauthor{Gavin Brown}{grbrown@cs.washington.edu}

\icmlkeywords{Machine Learning, ICML}

\vskip 0.3in
]

\printAffiliationsAndNotice{}

\begin{abstract}
    
We provide an improved analysis of standard differentially private gradient descent for linear regression under the squared error loss. Under modest assumptions on the input, we characterize the distribution of the iterate at each time step.

Our analysis leads to new results on the algorithm's accuracy: for a proper fixed choice of hyperparameters, the sample complexity depends only linearly on the dimension of the data. This matches the dimension-dependence of the (non-private) ordinary least squares estimator as well as that of recent private algorithms that rely on sophisticated adaptive gradient-clipping schemes \cite{varshney2022nearly,liu2023near}.

Our analysis of the iterates' distribution also allows us to 
construct confidence intervals for the empirical optimizer which adapt automatically to the variance of the algorithm on a particular data set.
We validate our theorems through experiments on synthetic data.
\end{abstract}

\section{INTRODUCTION}\label{sec:introduction}

Machine learning models trained on personal data are now ubiquitous—keyboard prediction models~\cite{xu2023federated}, sentence completion in email~\cite{differential_privacy_dp_fine_tuning}, and photo labeling~\cite{kurakin2022toward}, for example. Training  with \textit{differential privacy} \cite{DMNS} gives a strong guarantee that the model parameters reveal little about any single individual. However, differentially private algorithms necessarily introduce some distortion into the training process. Understanding the most accurate and efficient training procedures remains an important open question, with an extensive line of research dating back 15 years  \cite{KLNRS,chaudhuri2011differentially}.

The distortion introduced for privacy is complex to characterize; recent work has thus also investigated how to provide confidence intervals and other inferential tools that allow the model's user to correctly interpret its parameters. Confidence intervals on parameters are critical for applications of regression in the social and natural sciences, where they serve to evaluate effects' significance.

In this paper, we address these problems for a fundamental statistical learning problem, least-squares linear regression. Specifically, we give a new analysis of a widely studied differentially private algorithm, \textit{noisy gradient descent} (DP-GD). This algorithm repeatedly computes the (full) gradient at a point, adds Gaussian noise, and updates the iterate using the noisy gradient. 

Our central technical tool is a new characterization 
of the distribution of the iterates of private gradient descent. Under the assumption that the algorithm does not clip any gradients, we show that the distribution at any time step can be written as a Gaussian distribution about the empirical minimizer (plus a small bias term). All together, the iterates are drawn from a Gaussian process. We apply this characterization in two ways: we derive tighter error bounds and prove finite-sample coverage guarantees for natural confidence intervals constructions.

Our main result 
shows that the algorithm converges to a nontrivial solution---that is, an estimate whose distance from the true parameters is a $o(1)$ fraction of the parameter space's diameter---using only $n = \tilde\Theta(p)$ samples  (omitting dependency on the privacy parameters) when the  features and errors are distributed according to a Gaussian. 
Until recently, all private algorithms for linear regression required $\Omega(p^{3/2})$ samples to achieve nontrivial bounds.
This includes previous analyses of gradient descent \cite{cai2021cost}.
Three recent papers have broken this barrier: the exponential-time approach of \citet{liu2022differential} and the efficient algorithms of \citet{varshney2022nearly} and subsequently \citet{liu2023near}.
The latter two algorithms are based on variants of private gradient descent that use adaptive clipping frameworks that complicate both the privacy analysis and implementation.
We discuss these approaches in more detail in Related Work.

Our characterization of the iterates' distribution 
suggests that confidence interval constructions for Gaussian processes should apply to our setting. We confirm this, in both theory and practice:
we 
formally analyze and empirically test methods for computing instance-specific confidence intervals (that is, tailored to the variability of the algorithm on this particular data set).
These intervals convey useful information about the noise for privacy: in our experiments, their width is roughly comparable to that of the textbook nonprivate confidence intervals for the population parameter.
Even beyond the settings of our formal analysis, we give 
 general heuristics that achieve good coverage experimentally. These confidence interval constructions come at no cost to privacy—they use the variability among iterates (and their correlation structure) to estimate the variability of their mean. While prior works~\cite{shejwalkar2022recycling,rabanser2023training} discuss methods for generating confidence intervals based on intermediate iterates from DP gradient descent, they fail to provide any meaningful coverage guarantees.

We perform extensive experiments on synthetic data, isolating the effects of dimension and gradient clipping.
To the best of our knowledge, these are the first experiments which show privacy ``for free'' in high-dimensional (that is, with $p \approx n$ and $p$ large) private linear regression.
We also demonstrate the practicality of our confidence interval constructions.

\subsection{Our Results}

\paragraph{Formal Guarantees for Accuracy} Our theoretical analysis operates in the following distributional setting.
\begin{defn}[Generative Setting]\label{def:generative_setting}
    Let $\theta^*\in \R^p$ be the true regression parameter satisfying $\norm{\theta^*}\le 1$.
    For each $i\in [n]$, let the covariate $\bfx_i$ be drawn i.i.d.~from $\calN(0,\I_p)$ and the response $y_i \gets \bfx_i^\dagger \theta^* + \xi_i$, for $\xi_i \sim\calN(0,\sigma^2)$.
\end{defn}
Formally, we establish a set of conditions on the data (Condition~\ref{condition:main}) under which our algorithm performs well.
These conditions are satisfied with high probability by data arising from the generative setting above.
We focus on this setting to simplify the analysis as much as possible.
In Appendix~\ref{sec:more_experiments} we present experiments on other distributions.

\begin{thm}[Informal]\label{thm:intro_main}
    Assume we are in the generative setting (Definition~\ref{def:generative_setting}).
    Assume $n=\tilde\Omega( p )$.
    Set clipping threshold $\gamma = \tilde\Theta(\sigma \sqrt{p})$, step size $\eta = O(1)$, and number of steps $T = \tilde{O}(1)$. With high probability the final iterate $\theta_T$ of Algorithm~\ref{alg:DPGD} satisfies
    \begin{align}
        \norm{\theta_T - \theta^*} \le \tilde O\bracket{\sqrt{\frac{\sigma^2 p}{n}} + \frac{\sigma p}{\sqrt{\rho}n}}.
    \end{align}
\end{thm}
Here the parameters are set so that Algorithm~\ref{alg:DPGD} satisfies $\rho$-zCDP (see Section~\ref{sec:gradient_analysis} and Appendix~\ref{sec:DP_prelims}).

Theorem~\ref{thm:intro_main} follows from combining Theorem~\ref{thm:main_accuracy_claim}, which analyzes the accuracy of DP-GD relative to the OLS solution, 
with Claim \ref{claim:random_design}, which comes from the work of \citet{hsu2011analysis} and says that $\hat\theta$ is close to $\theta^*$.

\paragraph{Formal Guarantees for Confidence Intervals}
From our theoretical results, we expect any two iterates sufficiently separated in time to be well-approximated by independent draws from the stationary distribution, which (we show) is a Gaussian centered at the empirical minimizer: $\calN(\hat\theta, c\cdot \bfA)$, where $c\in \R$ depends on the algorithm's hyperparameters (e.g., privacy budget) and fixed matrix $\bfA$ is approximately spherical; it depends only on the step size and the covariance of the data. We can control the quality of this approximation and provide provable, non-asymptotic guarantees for natural confidence interval constructions. 

\paragraph{Experiments}
Our empirical results provide a strong complement to our formal results.
They validate our theorems by showing that the predicted behavior occurs at practical sample sizes and privacy budgets. 
For example, Figure~\ref{fig:pn} demonstrates how our algorithm's error is approximately constant when the ratio $\frac p n$ is fixed, which mirrors the behavior of OLS and stands in contrast to algorithms requiring $\Omega(p^{3/2})$ examples.

Our empirical results also investigate settings not considered by our theorems.
We demonstrate the accuracy of DP-GD and the validity of our confidence intervals constructions on distributions beyond those in Definition~\ref{def:generative_setting}. 
One phenomenon that affects our methods is gradient clipping, which introduces bias in settings with irregular data (such as outliers). In such settings, DP-GD can be viewed as optimizing a different objective function (roughly, a Huberized loss) and that the confidence intervals we generate capture the minimum of this smoothed objective.

\begin{figure}[h]
\centering
\includegraphics[width=0.4\textwidth]{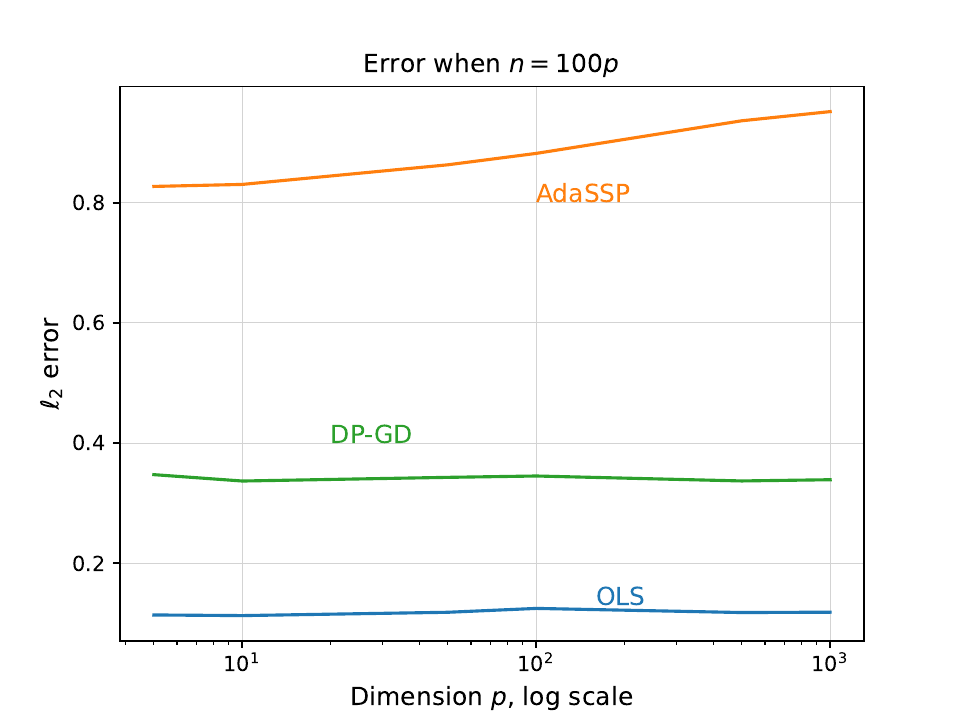}
\caption{We fix the ratio $p/n$ and let $p$ grow, with $\rho = 0.05$.
Run on data from a well-specified linear model, both DP-GD and OLS have constant error.
We compare with AdaSSP \cite{wang2018revisiting}, a popular algorithm that requires $n=\Omega(p^{3/2})$ examples.}
\label{fig:pn}
\end{figure}

\subsection{Techniques}

Our first technical tool characterizes the distribution of the iterates of DP-GD.
We start with the observation that, on any step where there is no clipping, the iterates satisfy the linear recurrence relation
\begin{align}
    \theta_{t} \gets \eta \cdot \Sigma \hat\theta + (\I-\eta\Sigma) \theta_{t-1} + \eta \cdot \bfz_{t-1},
\end{align}
where $\bfz_t$ is the noise added at time $t$ and $\Sigma = \frac 1 n \bfX^\dagger \bfX$ is the covariance.
We solve this recurrence and collect the noise terms, expressing $\theta_t$ as a Gaussian centered at $\hat\theta$ plus an exponentially decaying bias term:
\begin{align}
    \theta_t = \hat\theta + (\I-\eta\Sigma)^t (\theta_0-\hat\theta) + \eta \cdot \bfz_t',
\end{align}
where $\bfz_t'$ is a Gaussian random variable that depends on $\bfz_0,\ldots,\bfz_{t-1}$.

Under Definition~\ref{def:generative_setting}, a standard application of Cauchy--Schwarz implies that gradients are likely bounded by $\tilde{O}(p)$, so setting this as our clipping threshold $\gamma$ allows us to apply this distributional result.

Our main theoretical result is to show that, under the same distributional assumptions, the gradients are bounded by roughly $\sqrt{p}$, ignoring logarithmic factors and dependence on $\sigma$, the noise level.
Using $\gamma\approx \sqrt{p}$ as the clipping threshold allows us to add, at each iteration, privacy noise $\calN(0, c^2 \I)$ for $c^2 \approx \frac{p}{n^2}$ (ignoring privacy parameters). 
We show this bound holds on typical datasets and with high probability over the randomness of the algorithm.
On typical data the norm of $\bfx_i$ is roughly $\sqrt{p}$, so to control the norm of the gradient $\bfx_i(y_i-\bfx_i^\dagger\theta_t)$ it suffices to bound the absolute residual by $O(1)$.
We can show that this holds at initialization, using the fact that our initial iterate has no dependence on the data.  

We also expect the bound to hold after convergence. 
For an informal argument, consider a single update from $\theta^*$, where the gradient should nearly vanish.
We would move to $\theta_1 \gets \theta^* +\bfz$, where $\bfz \sim \calN(0,c^2 \I)$.
Since $c^2 \approx \frac{p}{n^2}$, we expect $\norm{\bfz} \approx c\sqrt{p} = O(1)$.
Plugging this in, we see
\begin{align}
    \abs{y_i - \bfx_i^\dagger \theta_1}
        &= \abs{(\bfx_i^\dagger\theta^* + \xi_i) - \bfx_i^\dagger(\theta^* + \bfz)} \\
        &= \abs{\xi_i - \bfx_i^\dagger \bfz}.
\end{align}
Since $\xi_i$ is drawn from a known distribution and $\bfx_i$ is independent of $\bfz$, we bound this residual with high probability.

The proof of Theorem~\ref{thm:main_accuracy_claim} formalizes these heuristics and shows that the same bound applies over every gradient step.

\

\vspace{-0.6cm}

\paragraph{Confidence Intervals}
The distributional form we show for the iterates suggests several natural methods for constructing confidence intervals for the parameter being estimated. 
In this work, we focus on confidence intervals for the \textit{empirical} minimizer—that is, the regression vector that minimizes the loss on the data set. 
Thus, our confidence intervals capture the uncertainty introduced by our privacy mechanism.
(It also makes sense to give confidence intervals for a population-level minimizer in settings where the data represent a random sample—we focus on the empirical parameter for simplicity.) We consider confidence intervals for a single coordinate of the parameter, since this is the most common use case.

We consider two approaches: one based on running the entire algorithm repeatedly and the other based on estimating the in-sequence variance of the stream of iterates. For this latter approach, we consider both a simple checkpointing strategy as well as a more data-efficient averaging strategy from the empirical process literature. These are discussed in Section~\ref{sec:CI-expt}.

\subsection{Limitations and Future Work}
Our work has several limitations, each of which presents natural directions for further exploration. 
First, some parts of our theoretical analysis require that the data be well-conditioned and have Gaussian-like concentration properties. 
Aside the exponential-time approach of~\citet{liu2022differential}, no private algorithms achieving $O(p)$ sample complexity incur a polynomial dependence on the condition number of the covariates \cite{varshney2022nearly,liu2023near}.
Removing this dependence remains a notable open problem.

We construct confidence intervals for the empirical minimizer of the loss function after clipping, which in general differs from the least-squares estimator (see Section~\ref{sec:effect_of_clipping}) and the population parameter. 
The former limitation is inherent (since the exact OLS solution has unbounded sensitivity) but an extension of our methods to population quantities would likely be useful.

Our experiments use synthetic data sets. 
A wider study of how these methods adapt to practical regression tasks is an important project but beyond the scope of this paper.

\subsection{Related Work}
\label{sec:related_work}

\paragraph{Private Linear Regression}

Under the assumption that the covariates are drawn from a subgaussian distribution and the responses arise from a linear model, the exponential-time approach of~\citet{liu2022differential} achieves nearly optimal error, matching a lower bound of~\citet{cai2021cost}.
In the remainder of this subsection, we survey a number of efficient approaches.

\paragraph{\it{Sufficient Statistics}}
A standard approach for private regression is \emph{sufficient statistics perturbation} \citep[see, e.g.,][]{vu2009differential,foulds2016theory,sheffet2017differentially,sheffet2019old}.
One algorithm which stands out for its practical accuracy and theoretical guarantees is the \emph{AdaSSP} algorithm of \citet{wang2018revisiting}, which relies on prior bounds for the covariates and labels.

To overcome the reliance on this prior knowledge, \citet{milionis2022differentially} build on algorithms of \citet{kamath2019privately} to give theoretical guarantees for linear regression with unbounded covariates.
More recently, \citet{tang2023improved} use AdaSSP inside a boosting routine.

As the dimension of the problem grows, these approaches suffer high error: accurate private estimation of $\bfX^\dagger \bfX$, which is necessary for SSP to succeed, requires $n=\tilde{\Omega}(p^{3/2})$ examples \cite{dwork2014analyze,kamath2022new,narayananbetter}.

An alternative approach is to view regression as an optimization problem, seeking a parameter vector that minimizes the empirical error.
Such algorithms form a cornerstone of the differential privacy literature.

Under assumptions similar to ours, the approach of~\citet{cai2021cost} achieves a near-optimal statistical rate via full-batch private gradient descent, where sensitivity of the gradients is controlled via projecting parameters to a bounded set. 
Their analysis, however, only applies when $n = \Omega(p^{3/2})$.
\citet{avella2021differentially} gave a general convergence analysis for private $M$-estimators (including Huber regression) but did not explicitly track the dimension dependence. 
Their approach bears similarities to gradient clipping, as we discuss in Section~\ref{sec:effect_of_clipping}.
The work of~\citet{varshney2022nearly} gave the first efficient algorithm for private linear regression requiring only $n=\tilde{O}(p)$ examples. 
Their approach uses differentially private stochastic gradient descent and an adaptive gradient-clipping scheme based on private quantiles.
Later,~\citet{liu2023near} gave a robust and private algorithm using adaptive clipping and full-batch gradient descent, improving upon the sample complexity of~\citet{varshney2022nearly} (by improving the dependence on the condition number of the design matrix).
Our approach is most similar to that of~\citet{liu2023near}; while their analysis involves adaptive clipping and strong \emph{resilience} properties of the input data, our algorithm and analysis are simpler.

A natural alternative strategy for private linear regression is to first clip the covariates and the responses and then run noisy \emph{projected} gradient descent, projecting each iterate into a constraint set. While this results in $\ell_2$-bounded gradient, the sample complexity of such an algorithm becomes $n=\Omega(p^{3/2})$~\cite{cai2021cost}.

In connection with robust statistics, a line of work gives private regression algorithms based on finding approximate medians~\cite{dwork2009differential, alabi2022differentially,sarathy2022analyzing,knop2022differentially,amin2022easy}. Informally, the algorithms solve the linear regression problem (or a robust variant) on multiple splits of the data and apply a consensus-based DP method (e.g., propose-test-release~\cite{dwork2009differential}, or the exponential mechanism~\cite{mcsherry2007mechanism}) to choose the regression coefficients. 
To the best of our knowledge, this class of approach cannot achieve the optimal sample complexity of $n=\widetilde O(p)$.

\paragraph{Private Confidence Intervals}
Some approaches generate confidence intervals using bootstrapping and related approaches~\cite{brawner2018bootstrap,wang2022differentially,covington2021unbiased}.

Several approaches arise naturally from sufficient statistics perturbation.
\citet{sheffet2017differentially} gave valid confidence intervals for linear regression.
The parametric bootstrap \citep{ferrando2022parametric} is also a natural choice in this setting where we already work under strong distribution assumptions.

Other approaches stem from the geometry of optimization landscape~\cite{wang2019differentially,avella2021differentially}; see also the non-private work of \citep{chen2020statistical}.

In a recent line of work,~\cite{shejwalkar2022recycling,rabanser2023training} use multiple checkpoints from a single run of DP-GD~\cite{BST14,DP-DL} to provide confidence intervals for predictions. While there is some algorithmic similarity to our procedures, these works do not provide rigorous parameter confidence interval estimates.

\section{Private Gradient Descent for Regression}
\label{sec:gradient_analysis}

\paragraph{Notation}
We use lowercase bold for vectors and uppercase bold for matrices, so $(\bfX,\bfy)$ is a data set and $(\bfx_i,y_i)$ a single observation.
Special quantities receive Greek letters: $\theta\in\R^p$ denotes a regression vector and $\cov=\frac{1}{n}\bfX^\dagger \bfX$, the empirical covariance.
We ``clip'' vectors in the standard way: $\mathtt{CLIP}_\gamma(\bfx) = \bfx \cdot\min\{1, \gamma/\norm{\bfx}\}$.

\subsection{Algorithm}

Algorithm~\ref{alg:DPGD} is differentially private gradient descent. 
Similar to the more complex linear regression algorithm of \citet{liu2023near}, it controls the sensitivity by clipping individual gradients: a full-batch version of the widely used private \emph{stochastic} gradient descent \cite{DP-DL}.
This is in contrast to approaches which rely on projecting the parameters to a convex set \citep[see, e.g.,][]{BST14}.

We present our privacy guarantee with \emph{(zero-)concentrated differential privacy} (zCDP) \cite{dwork2016concentrated,bun2016concentrated}.
For the definition of zCDP and basic properties, including conversion to $(\eps,\delta)$-DP, see Appendix~\ref{sec:DP_prelims}.
The guarantee comes directly from composing the Gaussian mechanism.
\begin{lem}
\label{lem:privacy}
For any noise variance $\lambda^2 \ge \frac{2 T \gamma^2}{\rho n^2}$, Algorithm 1 satisfies $\rho$-zCDP.
\end{lem}

\begin{algorithm}
\caption{DP-GD, $\calA(\bfX,\bfy; \gamma, \lambda,\eta, T, \theta_0)$}
\label{alg:DPGD}
\begin{algorithmic}[1]
\STATE \textbf{Input:} data $(\bfX,\bfy)\in \R^{n\times p}\times \R^n$;  clipping threshold $\gamma>0$; noise scale $\lambda>0$; step size $\eta>0$; number of iterations $T\in\mathbb{N}$; initial vector $\theta_0\in \R^p$
\FOR{$t=1,\ldots, T$}
\STATE $\bar{\bfg}_t \gets \frac 1 n \sum_{i=1}^n \mathtt{CLIP}_\gamma( -\bfx_i (y_i - \bfx_i^\dagger\theta_{t-1} ))$ 
\STATE Draw $\noise_t \sim \calN(0, \lambda^2 \I)$
\STATE $\theta_t \gets \theta_{t-1} - \eta\cdot  \bar{\bfg}_t + \eta \cdot\noise_t$
\ENDFOR
\STATE \textbf{Output:} $\theta_1,\ldots,\theta_T$.
\end{algorithmic}
\end{algorithm}

\subsection{Convergence, With and Without Clipping}

Throughout the paper, we deal with the exact distribution over the iterates of DP-GD.
As we show, if we remove the clipping step (i.e., set $\gamma=+\infty$), this distribution is exactly Gaussian.
This algorithm appears in many contexts under different names, including \emph{Noisy GD} and the \emph{(Unadjusted) Langevin Algorithm}.
Clipping ensures privacy, even if a particular execution of the algorithm does not clip any gradients.
We might hope to condition on the event that Algorithm~\ref{alg:DPGD} clips no gradients.
However, conditioning changes the output distribution.
We require a more careful approach.

We consider a \emph{coupling} between two executions of Algorithm~\ref{alg:DPGD}: one with clipping enforced and the other with $\gamma=+\infty$.
Of course, if Algorithm~\ref{alg:DPGD} receives an input where clipping occurs with high probability, then the output distributions of the two executions may differ greatly.
When clipping in the first case is unlikely, we can connect the output distributions between the two executions.

\newcommand{\couplingstatement}{
    Fix  data set $\bfX,\bfy$ and hyperparameters $\gamma, \lambda, \eta, T$, and $\theta_0$.
    Define random variables $\calO_\gamma$ and $\calO_\infty$ as
    \begin{align}
        \calO_\gamma &= \calA(\bfX,\bfy; \gamma, \lambda,\eta,T,\theta_0) \\
        \calO_\infty &= \calA(\bfX,\bfy; \infty, \lambda,\eta,T,\theta_0).
    \end{align}
    
    If Algorithm~\ref{alg:DPGD} on input $(\bfX,\bfy;\gamma,\lambda,\eta,T,\theta_0)$ clips nothing with probability $1-\beta$, 
    then $TV(\calO_\gamma,\calO_\infty)\le \beta$.
}
\begin{lem}[Coupling DP-GD without Clipping]\label{lem:couple_DPGD_noclipping}
    \couplingstatement
\end{lem}
Appendix~\ref{sec:proofs} contains details on couplings and the simple proof of this lemma, which uses the coupling induced by sharing randomness across runs of the algorithm.

The following claim characterizes the output of Algorithm~\ref{alg:DPGD} with $\gamma=+\infty$.
We see the distribution is Gaussian and centered at the empirical minimizer plus a bias term that goes to zero quickly as $t$ grows.
We define a matrix $\bfD = (\I-\eta \cov)^2$, where $\cov$ is the empirical covariance matrix and $\eta$ the step size.

The proof, which appears in Appendix~\ref{sec:proofs} only relies on the fact that the loss function is quadratic.
\newcommand{\distributionNoClippingSetup}{
    Fix a data set $(\bfX,\bfy)$ and step size $\eta$, noise scale $\lambda$, number of iterations $T$, and initial vector $\theta_0$.
    Define matrices $\cov = \frac 1 n \bfX^\dagger \bfX$ and $\bfD = (\I-\eta \cov)^2$; assume both are invertible.
    Let $\hat\theta = (\bfX^\dagger\bfX)^{-1}\bfX^\dagger \bfy$ be the least squares solution.
    Consider $\calA(\bfX,\bfy; \infty,\lambda,\eta,T,\theta_0)$, i.e., Algorithm~\ref{alg:DPGD} without clipping.
}
\newcommand{\distributionNoClippingShortConclusion}{
    For any $t\in[T]$, we have
    \begin{align}
        \theta_t = \hat\theta + (\I - \eta \cov)^t (\theta_0 - \hat\theta) + \eta \cdot\noise_t'
    \end{align}
    for $\noise_t' \sim\calN(0,\lambda^2 \bfA^{(t)})$ with $\bfA^{(t)} = (\I-\bfD)^{-1}(\I-\bfD^t)$.
}
\newcommand{\distributionNoClippingLongConclusion}{
    For any $t\in[T]$, we have
    \begin{align}
        \theta_t = \hat\theta + (\I - \eta \cov)^t (\theta_0 - \hat\theta) + \eta \sum_{i=1}^t (\I - \eta \cov)^{i-1} \noise^{t-i}.
    \end{align}
    This is equal in distribution to 
    \begin{align}
        \theta_t = \hat\theta + (\I - \eta \cov)^t (\theta_0 - \hat\theta) + \eta \cdot\noise_t'
    \end{align}
    for $\noise_t' \sim\calN(0,\lambda^2 \bfA^{(t)})$ with $\bfA^{(t)} = (\I-\bfD)^{-1}(\I-\bfD^t)$. 
}
\begin{lem}\label{lem:distribution_of_DPGD_no_clipping}
    \distributionNoClippingSetup
    \distributionNoClippingShortConclusion
\end{lem}
Note that $\bfz_t'$ depends on all the noise vectors up to time $t$.
We must take care when applying this lemma across the same run, as $\bfz_{t_1}'$ and $\bfz_{t_2}'$ are dependent random variables.

\subsection{Conditions that Ensure No Clipping}

To apply Lemma~\ref{lem:distribution_of_DPGD_no_clipping}, we need to reason about when no gradients are clipped.
To that end, we now define a deterministic  ``\noclippingcondition'' under which DP-GD clips no gradients with high probability.
This is a condition on data sets $(\bfX,\bfy)$ that is defined in terms of a few hyperparameters.
We will often leave these hyperparameters implicit, saying ``data set $(\bfX,\bfy)$ satisfies the \noclippingcondition'' instead of ``data set $(\bfX,\bfy)$ satisfies the \noclippingcondition with values $\sigma,\eta,T$, and $\multiplier.$''

Later, we will show that the \noclippingcondition is satisfied with high probability under distributional assumptions.
Conditions~\ref{condition:spectrum}-\ref{condition:response_noise_magnitude} follow from standard concentration statements.
Conditions~\ref{condition:true_param_direction} and \ref{condition:noise_covariate_independence} are less transparent; they capture notions of independence between $\theta^*, \{\bfx_i\},$ and $\{y_i-\bfx_i^\dagger \theta^*\}$.

\begin{condition}[\noclippingcondition]\label{condition:main}
    Let $\sigma,\eta$ and $\multiplier$ be nonnegative real values and let $T$ be a natural number.
    Let $(\bfX,\bfy)\in \R^{n\times p}\times\R^p$ be a data set.
    Define $\cov = \frac 1 n \bfX^\dagger \bfX$.
    There exists a $\theta\in\R^p$ such that, for all $i\in[n]$ and $t\in[T]$,
    \begin{enumerate}[(i)]
        \item $\frac 1 2 \I \preceq \cov \preceq 2\I$, $\norm{\I-\eta \cov} \le \frac 7 8$,
            \label{condition:spectrum}
        \item $\norm{\bfx_i}\le \multiplier \sqrt{p}$, \label{condition:covariate_norm}
        \item $\abs{y_i-\bfx_i^\dagger\theta} \le \multiplier  \sigma$, \label{condition:response_noise_magnitude}
        \item $\abs{\bfx_i^\dagger (\I - \eta \cov)^{t} \theta} \le \multiplier$ and \label{condition:true_param_direction}
        \item $\abs{\sum_{j=1}^n (y_j-\bfx_j^\dagger\theta) \cdot \bfx_i^\dagger A_t \bfx_j }\le \multiplier\sigma \sqrt{n p}$, where $A_t = (\I - (\I -\eta \cov)^t)\cov^{-1} $. \label{condition:noise_covariate_independence}
    \end{enumerate}
\end{condition}

The definition requires the existence of some $\theta$ with certain properties.
Informally, the reader should think of this as $\theta^*$, the ``true'' parameter.
Crucially, however, the definition does not require the existence of an underlying distribution.
We prove Lemma~\ref{lem:no_clipping} in Appendix~\ref{sec:proofs}.

\newcommand{\noClipping}{
    Fix nonnegative real numbers $\sigma, \eta$, and $\multiplier$.
    Fix natural number $T$.
    Assume data set $(\bfX,\bfy)\in\R^{n\times p}\times\R^n$ satisfies the \noclippingcondition (Condition~\ref{condition:main}) with values $(\sigma, \eta,\multiplier,T)$.
    
    Fix nonnegative real numbers $\gamma$ and $\lambda$.
    Consider running Algorithm~\ref{alg:DPGD}, i.e., $\calA(\bfX,\bfy; \gamma,\lambda,\eta,T,0)$ with $\theta_0 =0$ the initial point.
    Assume $\gamma\ge 4 \multiplier^2 \sigma \sqrt{p}$.
    For any $\beta\in(0,1)$, if
        $\frac{\gamma}{\eta \lambda} \ge 64 \multiplier^2 p \sqrt{\ln 2 n T/\beta},$
    then with probability $1-\beta$ Algorithm~\ref{alg:DPGD} clips no gradients.
}

\begin{lem}[No Clipping Occurs]\label{lem:no_clipping}
    \noClipping
\end{lem}

\subsection{Analyzing Algorithm~\ref{alg:DPGD}}

Recall the distributional setting we discussed in the introduction: 
there is some true parameter $\theta^*$ with $\norm{\theta^*}\le 1$ and observations are generated by drawing covariate $\bfx_i\sim\calN(0,\I)$ and setting response $y_i\gets \bfx_i^\dagger \theta^* + \xi_i$ for $\xi_i\sim \calN(0,\sigma^2)$.
In this section, we first show that data sets generated in this way satisfy the \noclippingcondition with high probability.
This implies that Algorithm~\ref{alg:DPGD}, with high probability, does not clip any gradients.

\newcommand{\assumptionHolds}{
    Fix data set size $n$ and data dimension $p\ge 2$.
    Fix $\theta^*$ with $\norm{\theta^*}\le 1$, let covariates $\bfx_i$ be drawn i.i.d.~from $\calN(0,\I)$, and responses $y_i = \bfx_i^\dagger \theta^* + \xi_i$ for $\xi_i\sim \calN(0,\sigma^2)$.
    Fix the step size $\eta = \frac 1 4$ and a natural number $T$. 
    There exists a constant $c$ such that, for any $\beta\in (0,1)$, if $n\ge c(p + \ln 1/\beta)$ then with probability at least $1-\beta$ data set $(X,y)$ satisfies the \noclippingcondition (Condition~\ref{condition:main}) with $\multiplier = 12 \ln^{1.5}(5nT/\beta)$.
}
\begin{lem}\label{lem:assumption_holds}
    \assumptionHolds
\end{lem}

\newcommand{\mainAccuracyClaim}{
    Fix $\theta^*\in \R^p$ with $p\ge 2$ and $\norm{\theta^*}\le 1$, let $n$ covariates $\bfx_i$ be drawn i.i.d.~from $\calN(0,\I)$ and responses $y_i =\bfx_i^\dagger \theta^* + \xi_i$ for $\xi_i\sim\calN(0,\sigma^2)$ for some fixed $\sigma$.
    
    Fix $\rho\ge 0$ and $\beta\in(0,1)$.
    Consider running Algorithm~\ref{alg:DPGD}, i.e., $\calA(\bfX,\bfy; \gamma, \lambda, \eta, T, 0)$ with step size $\eta= \frac 1 4$, initial point $\theta_0=0$, and, 
    for some absolute constant $c$,
    \begin{align}
        T &= c \log \frac{n\rho}{p}, \quad \lambda^2 = \frac{2T\gamma^2}{\rho n^2}, \text{ and}\\
        \gamma &= c \sigma \sqrt{p}\log^3 \bracket{\frac{nT}{\beta}}. 
    \end{align}
    Recall $\hat\theta$ the OLS solution.
    If  $n \ge c\bracket{p + \sqrt{p}\log^4 \rho/\beta}$,
    then with probability at least $1-\beta$ Algorithm~\ref{alg:DPGD} returns a final iterate $\theta_T$ such that, for some constant $c'$,
    \begin{align}
        \norm{\hat\theta - \theta_T}\le c'\ln^4(n \rho/\beta p)\cdot \frac{\sigma p }{\sqrt{\rho}n}.
    \end{align}
    
}

We combine this statement with Lemma~\ref{lem:no_clipping}, which says that clipping is unlikely under the \noclippingcondition, and Lemma~\ref{lem:distribution_of_DPGD_no_clipping}, which characterizes the output distribution of DP-GD when there is no clipping.
We prove this theorem in Appendix~\ref{sec:proofs}.
\begin{thm}[Main Accuracy Claim]\label{thm:main_accuracy_claim}
    \mainAccuracyClaim
\end{thm}
Recall from Lemma~\ref{lem:privacy} that this setting of $\lambda$ is exactly what we need to achieve $\rho$-zCDP for Algorithm~\ref{alg:DPGD}.

\subsection{Characterizing the Effect of Clipping}\label{sec:effect_of_clipping}

In  Sections~\ref{sec:gradient_analysis} and \ref{sec:cis_guarantees}, we analyze DP-GD when there is no clipping.
However, the optimization problem remains well-specified, even under significant clipping

\citet{song2020characterizing} showed that for generalized linear losses, the post-clipping gradients correspond to a different convex loss, which for linear regression relates to the well-studied \emph{Huber loss}. For parameter $B>0$, define
$$
    \ell_B(\theta;\bfx,y) = \begin{cases}
        \frac{1}{2} (y_i-\ip{\bfx_i}{\theta})^2, \quad \text{if $\abs{y_i-\ip{\bfx_i}{\theta}}\le B$} \\
        B\bracket{\abs{y_i-\ip{\bfx_i}{\theta}} -\frac 1 2 B}, \quad \text{otherwise.}
    \end{cases}
$$
In our setting, the individual loss $\frac 1 2 (y_i-\ip{\bfx_i}{\theta})^2$ after clipping corresponds to the Huber loss with per-datum parameter $B_i=\frac{\gamma}{\norm{\bfx_i}}$. 
Compare with \citet{avella2021differentially}, who perform private gradient descent with a loss similar to $\ell_B(\theta;\bfx,y)\cdot \min\left\{1, \frac{2}{\norm{\bfx}^2}\right\}$ for fixed $B$.

\section{Constructing Confidence Intervals} \label{sec:ci}

In this section we present three methods for per-coordinate confidence intervals and provide coverage guarantees for two.

\subsection{Methods for Confidence Intervals}

Each construction creates a list $\theta^{(1)},\ldots,\theta^{(m)}$ of parameter estimates and computes the sample mean
$\bar\theta=\frac 1 m \sum_{\ell=1}^m \theta^{(\ell)}$ and, for every $j\in [p]$, the sample variance $\hat\sigma_j^2=\frac{1}{m-1}\sum_{\ell=1}^m(\theta_j^{(\ell)}-\bar\theta_j)^2$.
The confidence interval is constructed as if the iterates came from a Gaussian distribution with unknown mean and variance:
\begin{align}
        \bar\theta_j \pm t_{\alpha/2,m-1}\cdot \frac{\hat\sigma_j}{\sqrt{m}}, \label{eq:CI_def}
\end{align}
where $t_{\alpha,m-1}$ denotes the $\alpha$-th percentile of the Student's $t$ distribution with $m-1$ degrees of freedom.

The methods, then, differ in how they produce the estimates.

\paragraph{Independent Runs} 
We run Algorithm \ref{alg:DPGD} $m$ times independently and take $\theta^{(\ell)}$ to be the final iterate of the $\ell$-th run.

This estimator is simple and, since the estimates are independent, easy to analyze.
However,  this method requires paying for $m$ burn-in periods. 

\paragraph{Checkpoints} 
We run Algorithm~\ref{alg:DPGD} for $mT$ time steps and take $\theta^{(\ell)}$ to be the $\ell T$-th iterate.

Our theorems show that 
if the checkpoints are sufficiently separated then they are essentially independent. 
While this approach pays the burn-in cost only once, it disregards information by using only a subset of the iterates. 

\paragraph{All Iterates/Batched Means} The empirical variance of batched means from a single run of $mT$ time steps.  
Formally, we 
separate the iterates of Algorithm~\ref{alg:DPGD} into $m$ disjoint batches each of length $T$ and set $\theta^{(\ell)}$ to the empirical mean of the $\ell$-th batch. 

This approach may make better use of the privacy budget but
poses practical challenges. The batch size needs to be large relative to the autocorrelation, but we also require  several batches (as a rule of thumb, at least $10$). 

In the next subsection, we provide formal guarantees for the first two methods. We note that our experiments use methods that are slightly more practical but less amenable to analysis, for instance replacing ``final iterate'' with an average of the final few iterates. For further discussion of these details, see Section~\ref{sec:simulation} and Appendix~\ref{sec:more_experiments}.

\subsection{Formal Guarantees}\label{sec:cis_guarantees}

To highlight the relevant aspects, in this section we state guarantees for DP-GD without clipping and under the assumption that the input data satisfies the \noclippingcondition.
We can replace these restrictions with a distributional assumption, as in  Section~\ref{sec:gradient_analysis}.

\begin{thm}[Coverage]\label{thm:CI_coverage}
    Fix a data set $(\bfX,\bfy)$, step size $\eta$, and noise scale $\lambda$.
    Fix $\alpha,\beta\in(0,1)$ and integers $m$, $T$.
    Assume the data satisfies Condition \ref{condition:main} with values $(\sigma,\eta,c_0,mT)$.
    Let $\hat\theta$ be the OLS solution.
    
    There exists a constant $c$ such that, if $T\ge c \log \frac{\sigma m p}{\eta \lambda \beta}$, then Equation~\eqref{eq:CI_def} is a $1-\alpha-\beta$ confidence interval\footnote{That is, with probability at least $1-\alpha-\beta$ the interval contains the parameter of interest.} for $\hat\theta_j$ when the parameter estimates $\{\theta^{(\ell)}\}_{\ell=1}^m$ are produced in either of the following ways:
    \begin{itemize}
        \item Independent Runs: repeat DP-GD $m$ times independently, each with no clipping and $T$ time steps.
        Let $\theta^{(\ell)}$ be the final iterate of each run.
        \item Checkpoints: run DP-GD with no clipping for $mT$ steps.
        Let $\theta^{(\ell)}$ be the $\ell T$-th iterate.
    \end{itemize}
\end{thm}

\section{Experiments}\label{sec:simulation}

We perform experiments to confirm and complement our theoretical results.
Unless otherwise mentioned, we generate data by drawing $\theta^*$ randomly from the unit sphere, $\bfx_i$ from $\mathcal{N}(0, \I)$, and $y_i = \bfx_i^\dagger \theta^{*} + \xi_i$, where $\xi_i \sim \mathcal{N}(0, 1)$. 
We set $p=10$, $\gamma=5\sqrt{p}$, and $\rho = 0.015$ unless otherwise stated.\footnote{This corresponds to an $(\epsilon, \delta)$-DP guarantee with $\epsilon=0.925$ and $\delta = 10^{-6}$, see Claim~\ref{claim:CDP_to_DP}.}
We include more details and results in Appendix~\ref{sec:more_experiments}.

\subsection{Error, Dimension, and a Bias-Variance Tradeoff}

We highlight how our algorithm's error depends linearly on $p$, in contrast to standard approaches that require $p^{3/2}$ examples.
Prior algorithms with formal accuracy analysis demonstrating this linear dependence limit experiments to modest ($p\approx 10$) dimensions \cite{varshney2022nearly,liu2023near}.
We show that our approach can achieve privacy ``for free:'' at reasonable sample sizes, the error from sampling dominates the error due to privacy.

We first explore the accuracy of DP-GD and how it depends on the dimension and sample size. 
Figure~\ref{fig:pn}, presented in the introduction, allows $p$ and $n$ to grow together with a fixed ratio (here, $n=100p$). We fix the total number of gradient iterations to $10$.  We report the $\ell_2$-error between the differentially private parameter estimates and the OLS estimates.
These results demonstrate that the error of DP-GD is constant when $p/n$ is constant. 
We compare with OLS and the well-known AdaSSP algorithm \cite{wang2018revisiting}.  
In contrast to the other two approaches, the error of AdaSSP grows with $p$ in this graph.
Its error depends on $\frac{p^{3/2}}{n}$, so in this plot the error grows as the square root of $p$.

Figure~\ref{fig:cost_of_privacy_low_privacy} fixes the dimension and allows the sample size to grow.
We see that the error due to privacy noise (i.e.,  $\norm{\theta_T-\hat\theta}$) falls off with $1/n$, while the error due to sampling (i.e., $\norm{\hat\theta-\theta^*}$) falls off with $1/\sqrt{n}$.
At larger sample sizes, the non-private error dominates.
This experiment has a larger privacy budget to highlight the effect; see Figure~\ref{fig:cost_of_privacy_high_privacy} in Appendix~\ref{sec:more_experiments} for additional results.

\begin{figure}[h]
\centering
\includegraphics[width=0.4\textwidth]{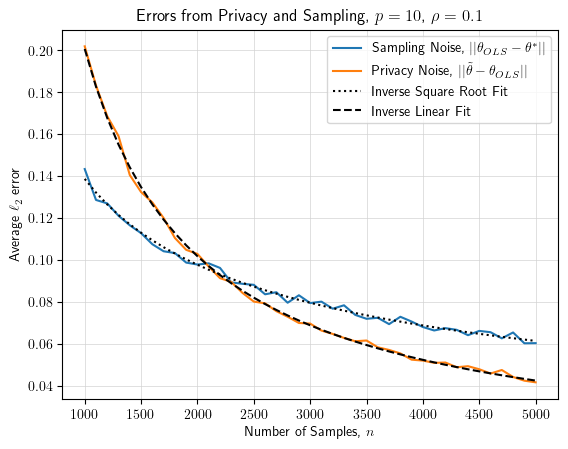}
\caption{The ``cost of privacy:'' fixing the dimension and allowing the sample size to grow, we see how the error due to sampling dominates the error from privacy. Each point is averaged over 100 independent trials.}
\label{fig:cost_of_privacy_low_privacy}
\end{figure}

The clipping threshold in these results was fixed to $\gamma=5\sqrt{p}$, as guided by our theory that no clipping occurs when $\gamma \gtrsim \sqrt{p}$.
Figure~\ref{fig:sqrtp} validates this theoretical result across dimensions.
We then move deliberately beyond this no-clipping regime in Figure~\ref{fig:bias_variance}, revealing a  bias-variance tradeoff and highlighting how the lowest error may occur under significant clipping.

\begin{figure}[h]
\centering
\includegraphics[width=0.4\textwidth]{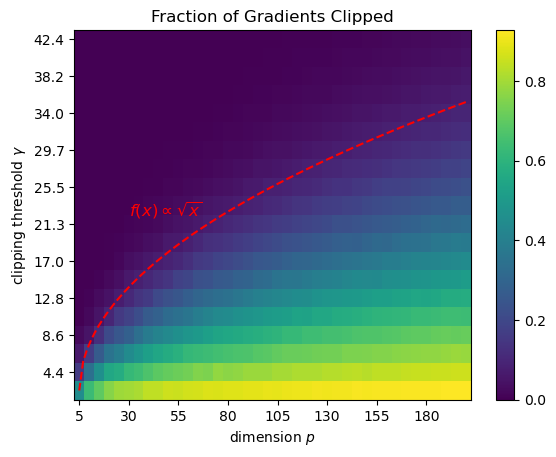}
\caption{
We see the fraction of gradients clipped over a grid on dimension and clipping threshold.
As the theory predicts, we see low clipping with $\gamma=\Omega(\sqrt{p})$.}
\label{fig:sqrtp}
\end{figure}

\begin{figure}[h]
\centering
\includegraphics[width=0.4\textwidth]{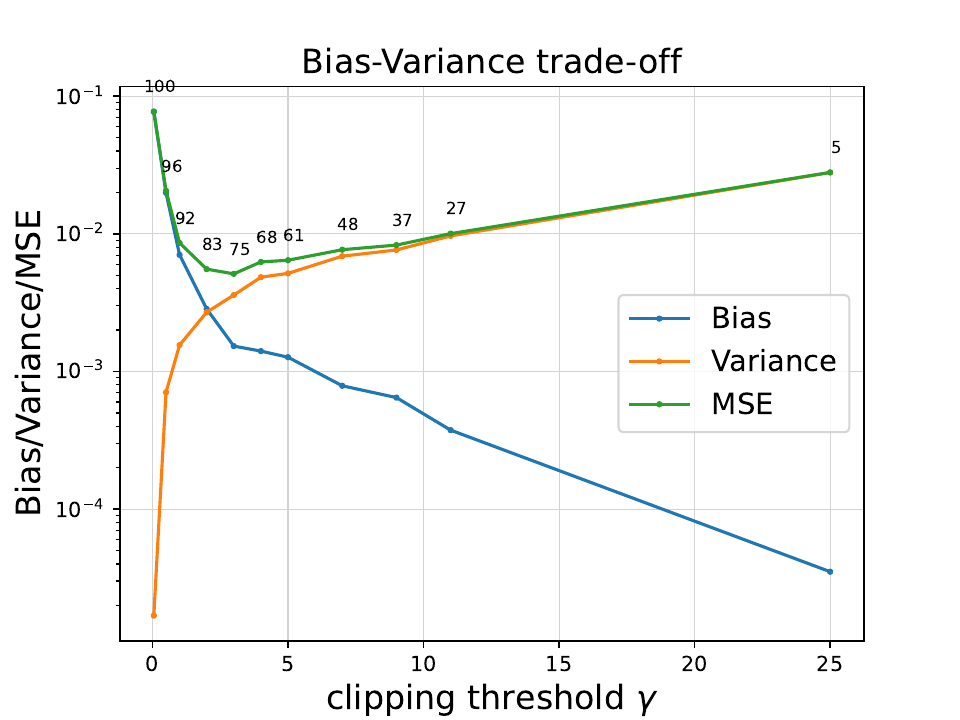}
\caption{We plot the error, (squared) bias, and variance of DP-GD as we change the clipping threshold. 
The numeric labels give the percentage of all gradients clipped. 
Low thresholds cause high clipping and bias, while high thresholds have little clipping but high variance.}
\label{fig:bias_variance}
\end{figure}

\subsection{Confidence Intervals}\label{sec:CI-expt}

Finally, we evaluate the three confidence interval constructions from Section~\ref{sec:ci}, comparing their empirical coverage and interval width.
We vary the total number of gradient iterations to clarify the regimes where each method performs well.
Using the notation from Section~\ref{sec:ci}, we use $m=10$ runs/checkpoints/batches and vary $T$.
We place the total number of gradient updates (i.e., the product $mT$) on the $x$ axis.
For more details, see Appendix~\ref{sec:more_experiments}.

Figure \ref{fig:coverage} shows the constructions' coverage properties as a function of the total number of gradients. Figure \ref{fig:length} shows the average length of the confidence intervals. We see that all  produce valid confidence intervals 
but the relative efficiency differs. With fewer iterations  the burn-in period is proportionally longer, so running the algorithm multiple times yields wider confidence intervals. 
In contrast, running the algorithm longer induces more auto-correlation between the iterates which means a larger batch-size is required to obtain valid intervals from the batched means approach. The checkpoints approach has a poor dependence on the number of iterations since a large fraction are disregarded.%

\begin{figure}[h]
\centering
\includegraphics[width=0.4\textwidth]{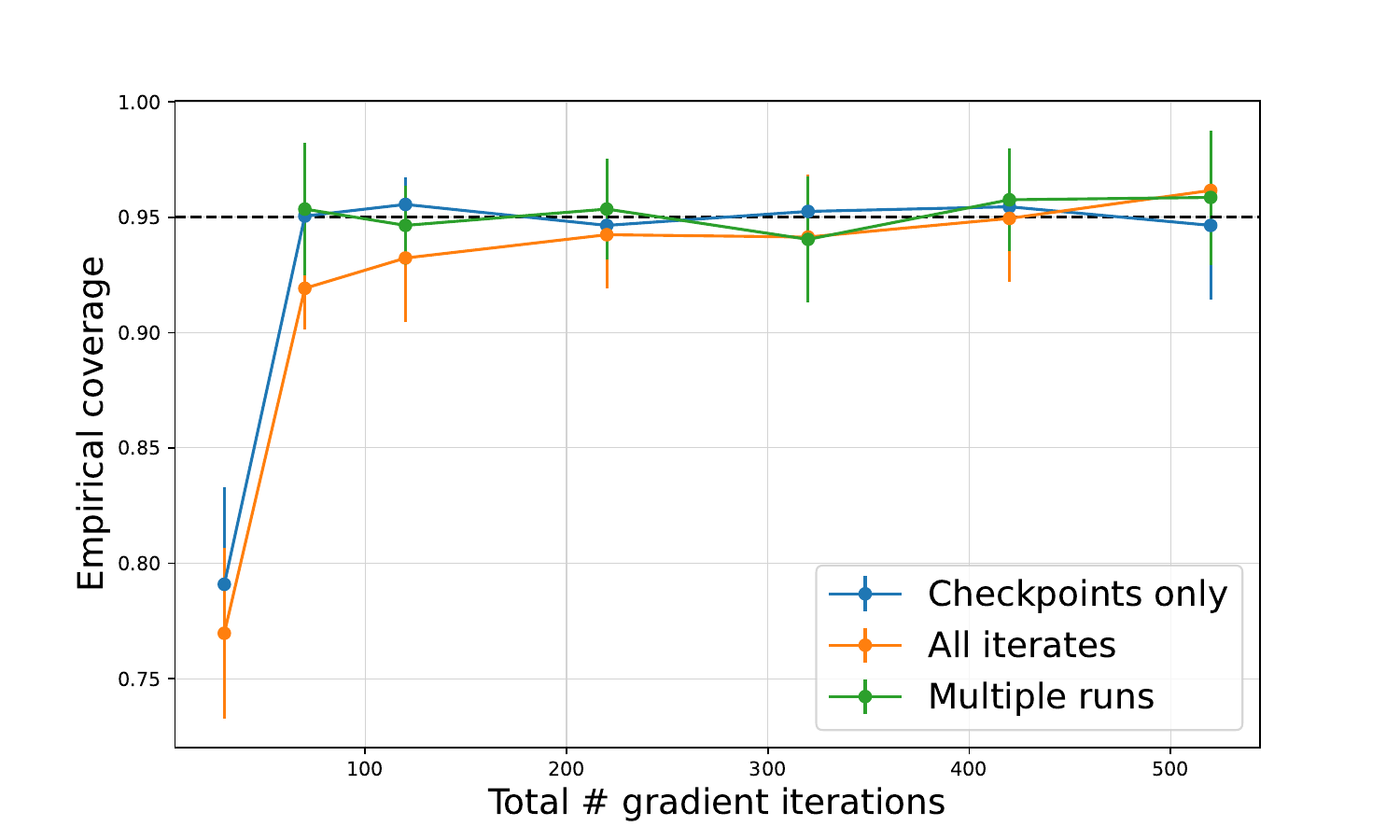}
\caption{Average empirical coverage across co-ordinates over 100 algorithm runs. Error bars reflect the $95$-percentiles of coverage across coordinates.}
\label{fig:coverage}
\end{figure}

\begin{figure}[H]
\centering
\includegraphics[width=0.4\textwidth]{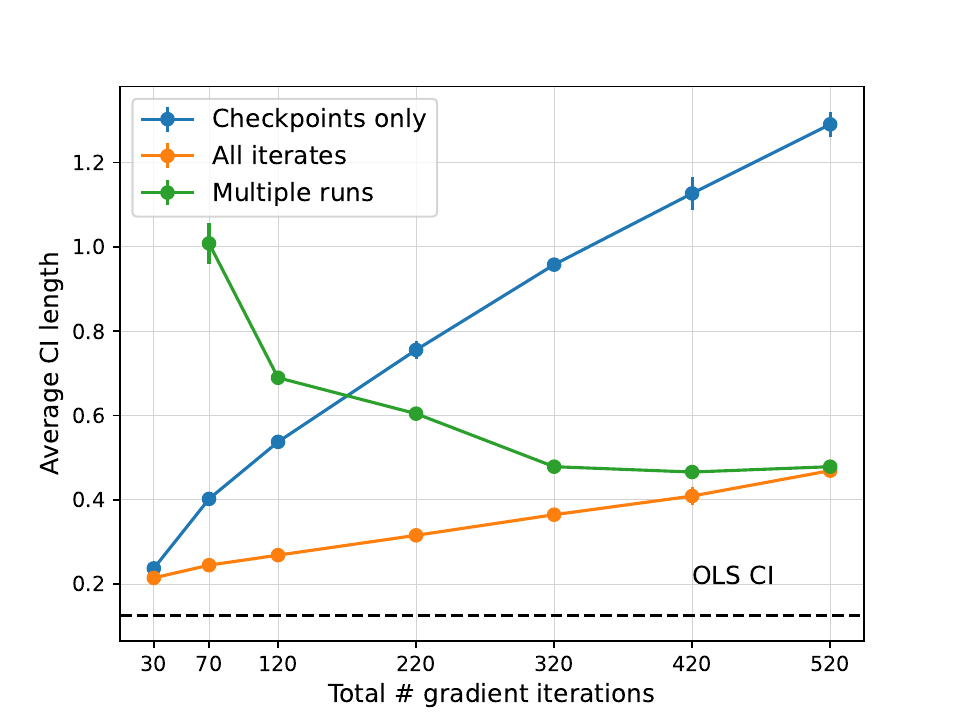}
\caption{Average confidence interval length for each construction, as the total number of gradient iterations increases. Error bars reflect $95$-th percentiles. For comparison, the dashed line shows the width of standard nonprivate OLS confidence intervals for the population quantity.}
\label{fig:length}
\end{figure}

\bibliography{reference}
\bibliographystyle{icml2024}

\newpage
\appendix
\onecolumn

\section{Preliminaries}
\label{sec:preliminaries}

\subsection{Differential Privacy}
\label{sec:DP_prelims}

We present definitions and basic facts about differential privacy.

\begin{defn}[Approximate Differential Privacy]
    For $\eps\ge 1$, and $\delta\in(0,1)$, a mechanism
    A mechanism $M:\calX^n\to \calY$ satisfies $(\eps,\delta)$-differential privacy if, for all $x,x'\in\calX^n$ that differ in one entry and all measurable events $E\subseteq \calY$, 
    \begin{align}
        \Pr[M(x)\in E]\le e^{\eps} \Pr[M(x')\in E]+\delta.
    \end{align}
\end{defn}
    
The guarantees we present in this work are in terms of \emph{(zero) concentrated differential privacy} \cite{dwork2016concentrated,bun2016concentrated}, a variant of differential privacy that allows us to cleanly express the privacy guarantees of DP-GD.

\begin{defn}[$\rho$-zCDP]
    A mechanism $M:\calX^n \to \calY$ satisfies \emph{$\rho$-zero concentrated differential privacy} ($\rho$-zCDP) if for all $x,x'\in\calX$ that differ in one entry and all $\alpha\in (1,\infty)$, we have $D_{\alpha}(M(x) \lVert M(x'))\le \rho \alpha$.
\end{defn}
\begin{defn}[R{\'e}nyi Divergence]
    For two distributions $p,w$ and $\alpha\ge 1$, the \emph{$\alpha$-R{\'e}nyi divergence} is $D_{\alpha}(p \lVert q) = \frac{1}{\alpha-1}\log \E_{y\sim q}\sqbracket{\bracket{\frac{p(y)}{q(y)}}^{\alpha}}$.
\end{defn}

The privacy guarantees for DP-GD follow composition plus the privacy guarantee for the standard multivariate Gaussian mechanism. 
zCDP allows us to cleanly express both.
\begin{claim}[Composition]\label{claim:privacy_composition}
    Suppose mechanism $M$ satisfies $\rho$-zCDP and mechanism $M'$ satisfies $\rho'$-zCDP.
    Then $(M,M')$ satisfies $(\rho+\rho')$-zCDP.
\end{claim}

\begin{claim}[Gaussian Mechanism]\label{claim:gaussian_mechanism}
    Let $q:\calX^n\to \calY$ satisfy, for all $x,x'\in\calX^n$ that differ in one entry, $\norm{q(x)-q(x')}_2\le \Delta$.
    Then, for any $\lambda>0$, the mechanism $M(x)=\calN(q(x), \lambda^2\I)$ is $\frac{\Delta^2}{2\lambda^2}$-zCDP. 
\end{claim}
We apply Claim~\ref{claim:gaussian_mechanism} to the average of gradients, where each gradient norm is clipped to $\gamma$.
The straightforward sensitivity analysis shows that we can set $\Delta = \frac{2\gamma}{n}$.

A mechanism satisfying $\rho$-zCDP also satisfies $(\eps,\delta)$-differential privacy.
In fact, it provides a continuum of such guarantees: one for every $\delta\in(0,1)$.
\begin{claim}\label{claim:CDP_to_DP}
    If $M$ satisfies $\rho$-zCDP, then $M$ satisfies $(\rho + 2\sqrt{\rho \log 1/\delta}, \delta)$-differential privacy for any $\delta>0$.
\end{claim}

\subsection{Linear Algebra}

\begin{fact}[Matrix Geometric Series]\label{fact:matrix_geometric_series}
    Let $\bfT$ be an invertible matrix. 
    Then $\sum_{j=0}^{n-1} \bfT^n = (\I - \bfT)^{-1} (\I - \bfT^n)$.
\end{fact}

\subsection{Concentration Inequalities}

\begin{claim}\label{claim:concentration_univariate_gaussian}
    If $x\sim \calN(0,\sigma^2)$ for some $\sigma^2 >0$, then $\Pr\sqbracket{\abs{x}\ge \sigma \sqrt{2 \ln 2/\beta} }\le \beta$.
\end{claim}

\begin{claim}\label{claim:concentration_of_gaussian_norm}
    Fix the number of dimensions $p$ and a PSD matrix $\Sigma$.
    Let $\bfx\sim \calN(0,\Sigma)$.
    For any $\beta\in(0,1)$, we have
    \begin{align}
        \Pr\sqbracket{\norm{\bfx} \ge \sqrt{\tr(\Sigma)} + \sqrt{2 \norm{\Sigma}\log 1/\beta}}\le \beta.
    \end{align}
\end{claim}

\begin{claim}[Concentration of Covariance]
\label{clm:concentration_covariance}
    Fix $\beta\in(0,1)$.
    Draw independent $\bfx_1,\ldots,\bfx_n\sim\calN(0,\I)$ and let $Z = \frac 1 n \sum_{i=1}^n \bfx_i\bfx_i^\dagger$.
    There exists a constant $c$ such that, if $n\ge c (p + \log 1/\beta)$, then with probability at least $1-\beta$ we have 
    $\frac 1 2 \mathbb{I} \preceq Z \preceq 2\mathbb{I}$.
\end{claim}

\begin{lem}
\label{lem:sub-exp_inner_uniform_dist}
Let $X$ be a uniform distribution on the unit sphere $\mathbb{S}^{p-1}$, and $z$ be any fixed unit vector.
Then we know the inner product $\langle X,z\rangle$ is sub-exponential.
Specifically, we have
\begin{align*}
    \Pr[|\langle X,z\rangle|\ge t]\le e^{-\frac{\sqrt{p}-1}{4}t}.
\end{align*}
\end{lem}

The classic analysis of least squares under fixed design establishes the convergence of the OLS estimator to the true parameter.
A nearly identical result holds under random design.
We state the result for the family of distributions we consider, but~\cite{hsu2011analysis} prove it for a much broader family of distributions.
\begin{claim}[Theorem 1 of~\cite{hsu2011analysis}]\label{claim:random_design}
    Let $\theta^*$ satisfy $\norm{\theta^*}\le 1$.
    Draw covariates $\bfx_1,\ldots,\bfx_n$ i.i.d.~from $\calN(0,\I_p)$ and let $y_i = \bfx_i^\dagger \theta^* + \xi_i$ for $\xi_i\sim\calN(0,\sigma^2)$.
    Let $\hat\theta$ be the OLS estimate.
    There exists constants $c, c'$ such that, if $n\ge c(p + \ln 1/\beta)$, then with probability at least $1-\beta$ we have
    \begin{align}
        \norm{\theta^* - \hat\theta}^2 \le \frac{c' \sigma^2 \bracket{p + \ln 1/\beta}}{n}.
    \end{align}
\end{claim}

\section{Experimental Details and Additional Results}
\label{sec:more_experiments}

\paragraph{Details on Confidence Interval Experiments}
Our experiments hold fixed the burn-in period to the first $20$ iterates.
We set $m$, the number of algorithm runs/checkpoints/batches to $10$ while varying the total number of gradient iterations.
(This is not always possible for multiple runs while holding fixed the total number of iterations, in which case we omit this approach from the figures.)
We vary the total number of iterations in this way to clarify the regimes where each method performs well.

\begin{figure}[h]
\centering
\includegraphics[width=0.4\textwidth]{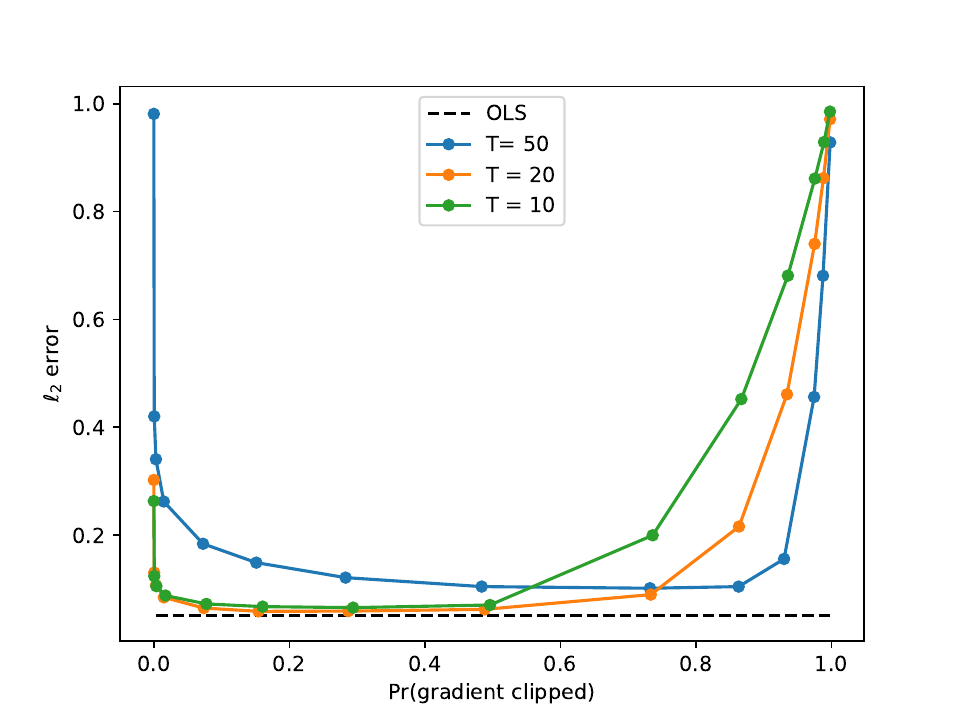}
\caption{The estimation error plotted against the proportion of gradients clipped. The right-hand side corresponds to bias from overly aggressive clipping. The left-hand side corresponds to variance from overly conservative clipping, which causes higher levels of noise for privacy.}
\label{fig:errorclip}
\end{figure}

In Figure \ref{fig:errorclip} we show how the rate of gradient clipping impacts $\ell_2$-error, and how this interacts with the total number of gradient iterations. One key insight is that error stays relatively low even when clipping $50\%$ of all gradients. As the number of gradient iterations grows, the tolerance to gradient clipping also seems to increase. 
Although our theoretical accuracy analysis proceeds by showing that no gradients are clipped, these experiments demonstrate that this stringent requirement is not necessary in practice.

\begin{figure}[h]
    \centering
    \begin{subfigure}[b]{0.4\textwidth}
        \centering
        \includegraphics[width=\textwidth]{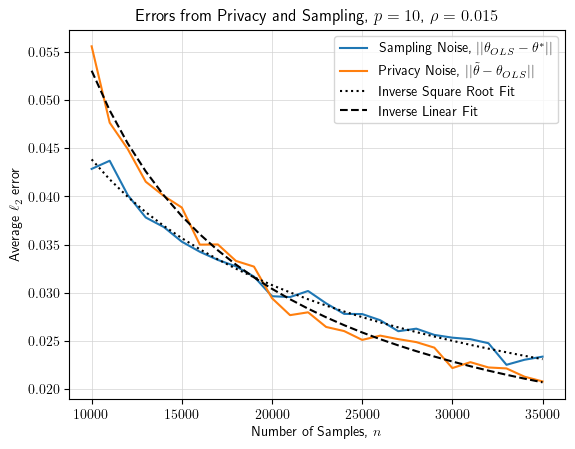}
        \caption{Experiment with $p=10$}
        \label{fig:cost_of_privacy_high_privacy_a}
    \end{subfigure}
    \hspace{1cm}
    \begin{subfigure}[b]{0.4\textwidth}
        \centering
        \includegraphics[width=\textwidth]{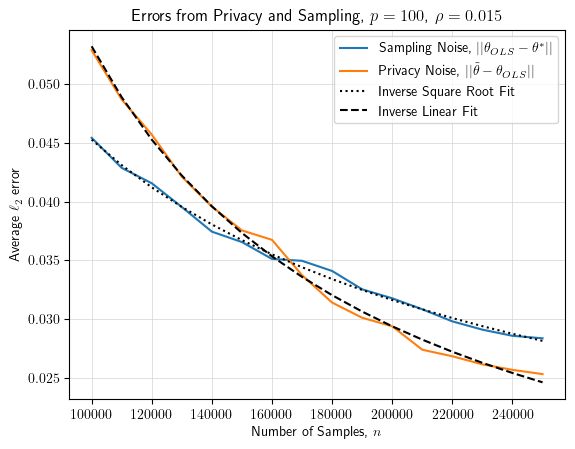}
        \caption{Experiment with $p=100$}
        \label{fig:cost_of_privacy_high_privacy_b}
    \end{subfigure}
    \caption{As in Figure~\ref{fig:cost_of_privacy_low_privacy}, we fix the dimension and let the sample size grow, seeing that the sampling error dominates the noise from privacy at reasonable sample sizes. These experiments are conducted with $\rho=0.015$, our standard setting. (a) uses $p=10$ and repeats each trial 100 times.
    (b) uses $p=100$ and repeats each trial 20 times, to reduce running time.}
    \label{fig:cost_of_privacy_high_privacy}
\end{figure}

\paragraph{Experiments with Anisotropic Data}
Our primary experiments are conducted on data sets where the covariates are drawn from a standard multivariate Gaussian distribution. We now move beyond this isotropic setting. 
In these experiments, for each dataset we first draw a random covariance matrix in the following way. 
We generate a random diagonal matrix $\Lambda$ with $\Lambda_{1,1}=2$, $\Lambda_{2,2}=1$, and, for all $3\le i \le p$, $\Lambda_{i,i}\sim\mathrm{Unif}([1,2])$ independently.
We then set the covariance $\Sigma = U \Lambda U^T$, where $U$ is a uniformly random rotation matrix.
The covariates are then sampled i.i.d.\ from $\calN(0,\Sigma)$; the remainder of the process is identical to the previous experiments.

\begin{figure}[h]
\centering
\includegraphics[width=0.4\textwidth]{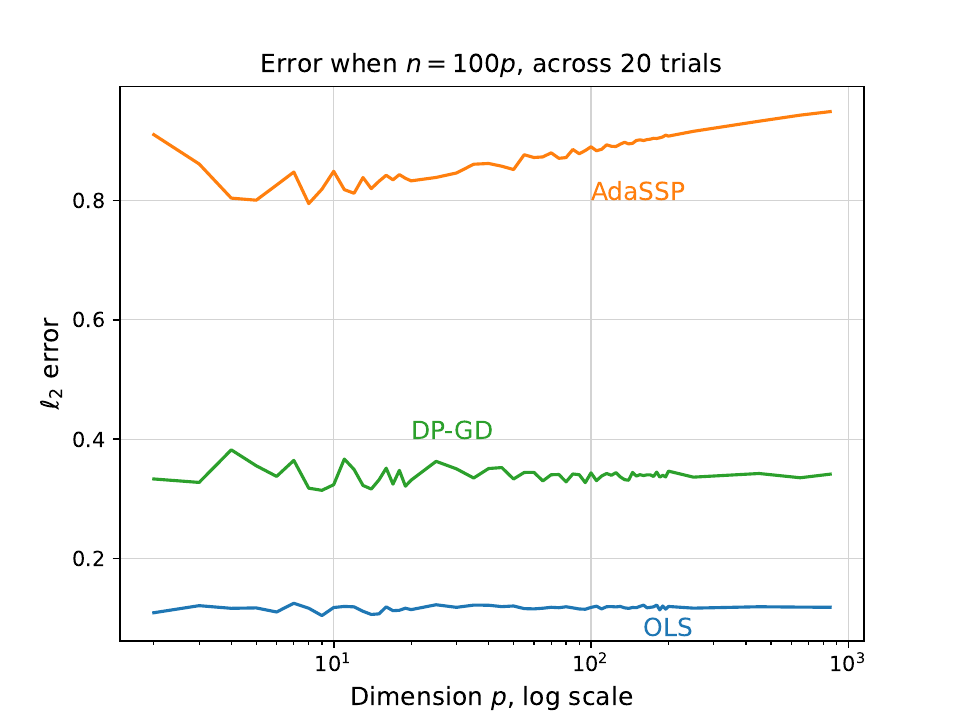}
\caption{We reproduce Figure \ref{fig:pn} with anisotropic data. We fix the ratio $p/n$ and let $p$ grow.
Run on data from a well-specified linear model, both DP-GD and OLS have constant error.
We compare with AdaSSP \cite{wang2018revisiting}, a popular algorithm that requires $n=\Omega(p^{3/2})$ examples.}
\label{fig:pn_ni}
\end{figure}

\begin{figure}[h]
\centering
\includegraphics[width=0.5\textwidth]{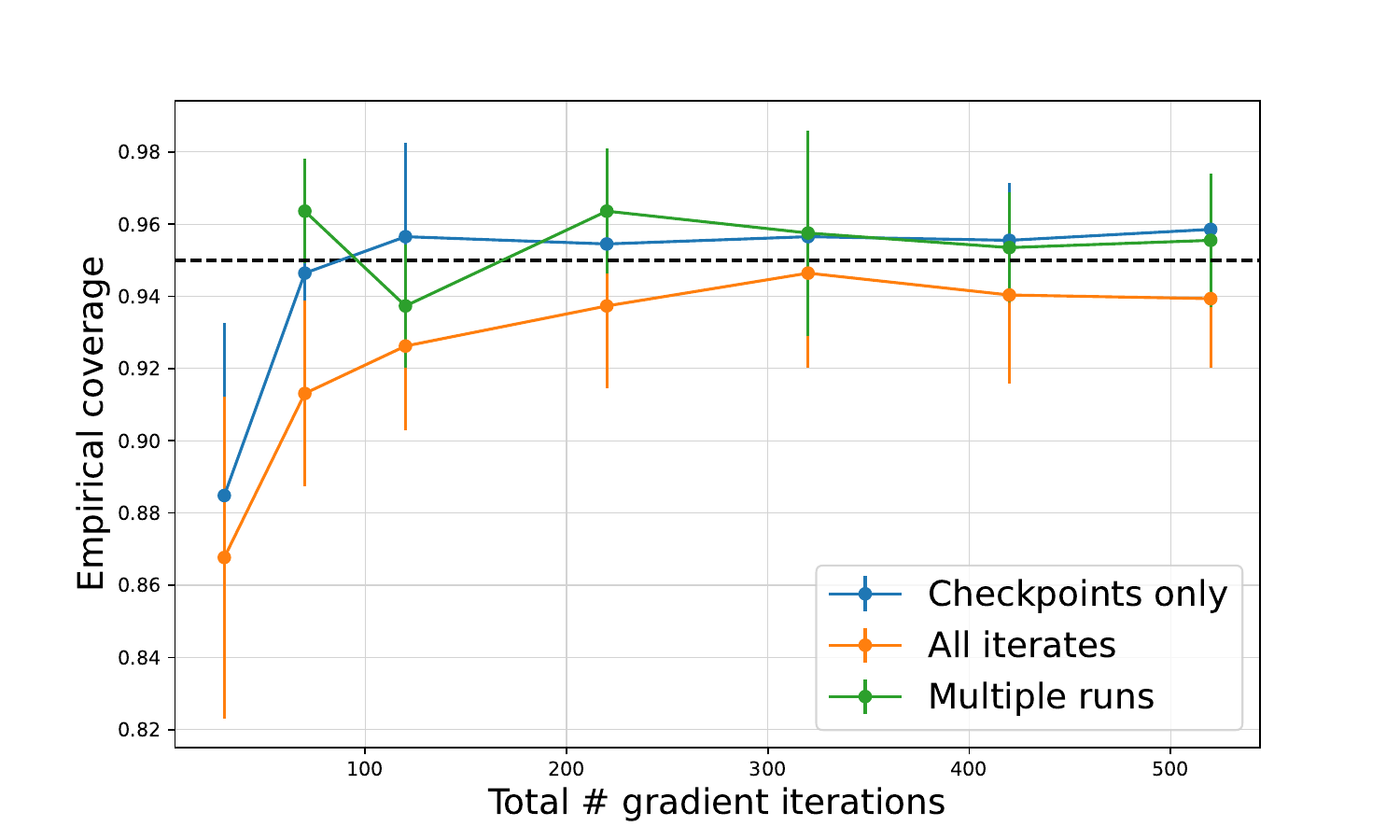}
\caption{We reproduce Figure \ref{fig:coverage} with anisotropic data. Average empirical coverage across co-ordinates over 100 algorithm runs. Error bars reflect the $95$-percentiles of coverage across coordinates.}
\label{fig:coverage_ni}
\end{figure}

\begin{figure}[h]
\centering
\includegraphics[width=0.5\textwidth]{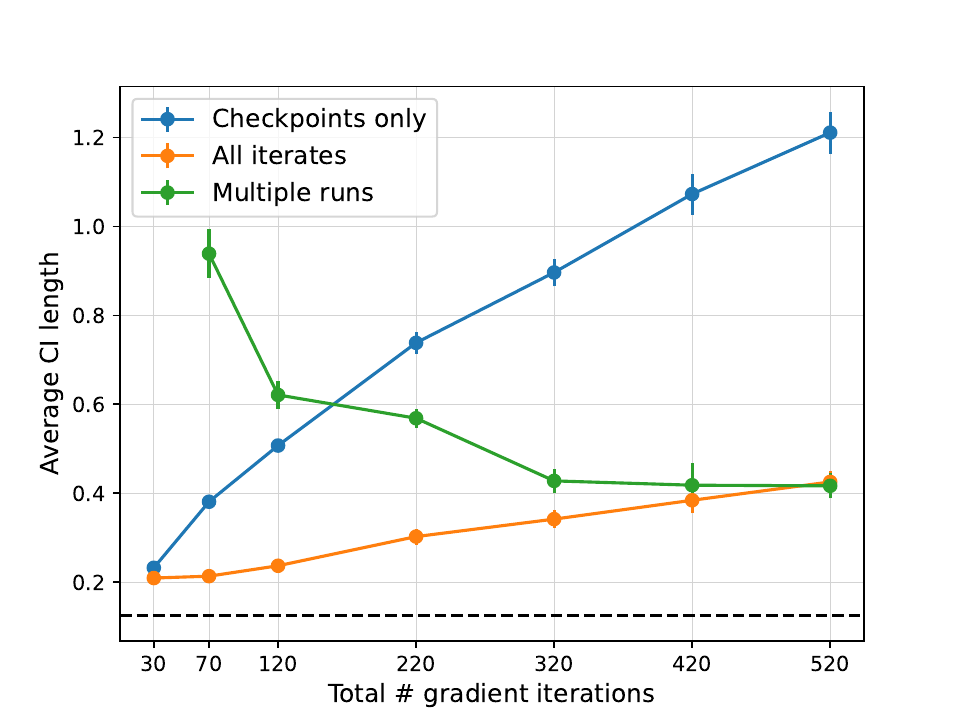}
\caption{We reproduce Figure \ref{fig:length} with anisotropic data. Average confidence interval length across co-ordinates for each confidence interval algorithm described in Section \ref{sec:ci} as the total number of gradient iterations increases. Error bar reflect the $95$-percentiles of coordinates CI length.}
\label{fig:length_ni}
\end{figure}

\section{Deferred Proofs}\label{sec:proofs}

\subsection{Clipping, Coupling, and Accuracy}

Before proving Lemma~\ref{lem:couple_DPGD_noclipping}, we define coupling and show how it relates to total variation distance.
\begin{defn}[Coupling]
    Let $p$ and $q$ be distributions over a space $\calX$.
    A pair of random variables $(X,Y)$ is called a \emph{coupling} of $(p,q)$ if, for all $x\in \calX$, $\Pr[X=x]=p(x)$ and $\Pr[Y=x]=q(x)$.
\end{defn}
Note that $X$ and $Y$ will not, in general, be independent.

\begin{claim}[Coupling and TV Distance]\label{claim:coupling_TV}
    Let $(X,Y)$ be a coupling of $(p,q)$.
    Then $TV(p,q)\le \Pr[X \neq Y]$.
\end{claim}

\begin{lem}[Restatement of Lemma~\ref{lem:couple_DPGD_noclipping}]
    \couplingstatement
\end{lem}
\begin{proof}
    We use the coupling induced by sharing randomness across the execution of the two algorithms.
    Let $\noise_1,\ldots,\noise_T$ be drawn i.i.d.\ from $\calN(0,\lambda^2\I)$.
    Note that this is the only randomness used by either algorithm.
    
    When the $\noise_i$ draws cause $\calA(\bfX,\bfy; \gamma,\lambda,\eta,T)$ to not clip, the two algorithms return the same output.
    Thus the probability two runs return \emph{different} outputs is at most $\beta$, which yields our bound on the total variation distance.
\end{proof}

We now prove our statement about the distribution of DP-GD without clipping.
The statement we prove also establishes a slightly more complicated expression for the noise.
\begin{lem}[Expanded Statement of Lemma~\ref{lem:distribution_of_DPGD_no_clipping}]\label{lem:expanded_dist_no_clipping}
    \distributionNoClippingSetup
    \distributionNoClippingLongConclusion
\end{lem}

\begin{proof}
    Algorithm~\ref{alg:DPGD}'s update step looks like
    \begin{align}
        \theta_{t+1} &\gets \theta_t - \eta \cdot \bar\bfg_t + \eta \cdot \noise_t,
            \label{eq:alg_update}
    \end{align}
    where $\noise_t\sim \calN(0,\lambda^2\I)$.
    Since there is no clipping, we have a closed form for $\bar{\bfg}_t$:
    \begin{align}
        \bar{\bfg}_t &= \frac{1}{n}\sum_{i=1}^n - \bfx_i (y_i - \bfx_i^\dagger \theta_t) \\
            &= \frac 1 n \sum_{i=1}^n \bfx_i \bfx_i^\dagger \theta_t - \frac{1}{n}\sum_{i=1}^n  \bfx_i y_i \\
            &= \frac 1 n \bfX^\dagger\bfX \theta_t - \frac 1 n \bfX^\dagger \bfy.
    \end{align}
    We simplify further and apply the fact that $\hat\theta = (\bfX^\dagger \bfX)^{-1} \bfX^\dagger \bfy$.
    We also plug in $\cov = \frac 1 n \bfX^\dagger \bfX$:
    \begin{align}
        \bar\bfg_t &= \frac 1 n \bfX^\dagger\bfX \theta_t - \frac 1 n \bfX^\dagger\bfX  \hat\theta 
            = \cov \theta_t - \cov \hat\theta.
    \end{align}
    
    Plugging this into Equation~\eqref{eq:alg_update}, the update formula, we have
    \begin{align}
        \theta_{t+1} \gets \eta\cdot \cov \hat\theta + (\I - \eta \cov)\theta_t + \eta \cdot \noise_t.
    \end{align}
    Solving the recursion, we arrive at a formula for $\theta_{t+1}$:
    \begin{align}
        \theta_{t+1} = (\I-\eta \cov)^t \theta_0 
                        + \sum_{i=1}^t (\I-\eta \cov)^{i-1} \eta \cov \hat\theta 
                        + \sum_{i=1}^t (\I-\eta \cov)^{i-1} \eta \cdot \noise_{t-i}.
                    \label{eq:dp_gd_recursion}
    \end{align}
    To simplify the second term in Equation~\eqref{eq:dp_gd_recursion}, apply the formula for matrix geometric series (Fact~\ref{fact:matrix_geometric_series}):
    \begin{align}
        \sum_{i=1}^t (\I-\eta \cov)^{i-1} = (\eta \cov)^{-1} (\I - (\I-\eta \cov)^t).
    \end{align}
    Thus the second term in Equation~\eqref{eq:dp_gd_recursion} is $(\I - (\I - \eta \cov )^t)\hat\theta = \hat{\theta} - (\I-\eta \cov)^t\hat{\theta}$.
    The first two terms together are $\hat{\theta} + (\I-\eta \cov)^t (\theta_0 - \hat{\theta})$.
    This establishes the first part of the claim.
    
    The final term in Equation~\eqref{eq:dp_gd_recursion}, corresponding to the noise for privacy, is slightly more involved.
    We use the independence of the noise and the fact that Gaussianity is preserved under summation.
    Continuing from the third term of Equation~\eqref{eq:dp_gd_recursion} and abusing notation to substitute distributions for random variables, we have
    \begin{align}
        &~\sum_{i=1}^t (\I-\eta \cov)^{i-1}\eta \cdot\noise_{t-i}\\
        &= \sum_{i=1}^t (\I-\eta \cov)^{i-1}\eta \cdot \calN(0,\lambda^2 \I) \\
            &= \eta \lambda \sum_{i=1}^t \calN(0,(\I-\eta \cov)^{2(i-1)}) \\
            &= \eta \lambda \cdot\calN\bracket{0,\sum_{i=1}^t (\I-\eta \cov)^{2(i-1)}} \\
            &= \eta \lambda \cdot \calN\bracket{0, (\I - (\I-\eta \cov)^2)^{-1} (\I - (\I-\eta \cov)^{2t})},
    \end{align}
    applying the formula for matrix geometric series (Fact~\ref{fact:matrix_geometric_series}) in the last line.
    This concludes the proof.
\end{proof}

\begin{lem}[Restatement of Lemma~\ref{lem:no_clipping}]
    \noClipping
\end{lem}
\begin{proof}
    We prove the claim by strong induction, relying on items (i-v) of the \noclippingcondition (Condition~\ref{condition:main}).
    Before beginning the induction, we prove a high-probability statement about the noise added for privacy.
    
    \textbf{Setup: Noise for Privacy}
    Recall that, at time $t$, we add independent noise $\noise_t \sim \calN(0, \lambda^2 \I)$.
    We show that, with probability at least $1-\beta$, we have for all $i\in [n]$ and $t_1,t_2\in [T]$ 
    \begin{align}
        \abs{\bfx_i^\dagger (\I-\eta \cov)^{t_1} \noise_{t_2}}\le 2 \multiplier \lambda \sqrt{p}\cdot \bracket{\frac 7 8}^{t_1}\sqrt{\ln 2nT/\beta}.
            \label{eq:noise_bound}
    \end{align}
    For any fixed covariates, we have 
    \begin{align}
        \bfx_i^\dagger (\I-\eta \cov)^{t_1} \noise_{t_2} \sim \lambda \cdot \calN\bracket{0, \norm{(\I-\eta \cov)^{t_1} \bfx_i }^2}.
    \end{align}
    With probability at least $1-\beta$ (simultaneously over all $i, t_1,$ and $t_2$) we have
    \begin{align}
        \abs{\bfx_i^\dagger (\I-\eta \cov)^{t_1} \noise_{t_2}} \le 2\lambda \norm{(\I-\eta \cov)^{t_1} \bfx_i} \cdot \sqrt{\ln 2nT/\beta}.
    \end{align}
    This holds independently of the value of the norm, so Conditions~\ref{condition:spectrum} and \ref{condition:covariate_norm} (and Cauchy--Schwarz repeatedly) proves Equation~\eqref{eq:noise_bound}.
    
    We now proceed to the induction. 
    Let $g_{i,t}= - \bfx_i (y_i - \bfx_i^\dagger \theta_t)$ be the gradient of the loss on point $i$ with respect to parameter $\theta_t$.
    Observe that $\norm{g_{i,t}} = \abs{y_i - \bfx_i^\dagger \theta_t}\cdot \norm{\bfx_i}$ and, furthermore, that Condition~\ref{condition:covariate_norm} promises $\norm{\bfx_i}\le \multiplier\sqrt{p}$.
    Thus, to show that the norm of the gradient is less than $\gamma$ we will show that the absolute residual is less than $\frac{\gamma}{\multiplier\sqrt{p}}$.
    
    \textbf{Base Case:} for $t=1$, we start from $\theta_0=0$, which means our residual is just $\abs{y_i}$.
    By the triangle inequality and Conditions~\ref{condition:response_noise_magnitude} and \ref{condition:true_param_direction}, we have
    \begin{align}
        \abs{y_i} = \abs{y_i - \bfx_i^\dagger\theta^* + \bfx_i^\dagger\theta^*} 
            &\le \abs{y_i - \bfx_i^\dagger\theta^*} +\abs{\bfx_i^\dagger\theta^*} \\
            &\le \multiplier \sigma + \multiplier.
    \end{align}
    This is less than $\frac{\gamma}{\multiplier\sqrt{p}}$ by assumption.
    
    \textbf{Induction Step:}
    Consider time $t+1$ and assume that no gradients have been clipped from the start through time $t$.
    From Lemma~\ref{lem:expanded_dist_no_clipping}, 
    we have a formula for the value of $\theta_{t}$:
    \begin{align}
        \theta_{t} = \hat\theta + (\I-\eta \cov)^{t}(\theta_0 - \hat\theta) + \eta \bfz_t'.
    \end{align}
    Plugging this equation into the formula for the residual at time $t$, we have (after adding and subtracting identical terms)
    \begin{align}
      \abs{y_i - \bfx_i^\dagger \theta_{t}}
            &= \abs{y_i - \bfx_i^\dagger\theta_t + \textcolor{teal}{\bfx_i^\dagger\theta - \bfx_i^\dagger\theta}} \\
            &= \abs{\bfx_i^\dagger (\theta - \theta_t) + (y_i - \bfx_i^\dagger\theta) } \\
            &=\abs{\bfx_i^\dagger \bracket{ \theta - \textcolor{teal}{\hat\theta - (\I-\eta \cov)^{t}(\theta_0 - \hat\theta) - \eta
            \bfz_t'}} + (y_i - \bfx_i^\dagger\theta)  }\\
            &=\abs{\bfx_i^\dagger \bracket{ \theta - \hat\theta - (\I-\eta \cov)^{t}(\theta_0 - \hat\theta+ \textcolor{blue}{\theta -\theta}) - \eta \bfz_t'} + (y_i - \bfx_i^\dagger\theta)},
    \end{align}
    where in the final line we added and subtracted $\theta$.
    We distribute terms and apply the triangle inequality, arriving at
    \begin{align}
        \abs{y_i - \bfx_i^\dagger \theta_t} &\le 
            \underbrace{\abs{\bfx_i^\dagger \bracket{\I - (\I-\eta \cov)^t} (\theta - \hat\theta)}}_{\mathrm{(i)}}
            + \underbrace{\abs{\bfx_i^\dagger (\I- \eta \cov)^t (\theta_0 -\theta)}}_{\mathrm{(ii)}}
            + \underbrace{\abs{\eta \cdot\bfx_i^\dagger \bfz_t'}}_{\mathrm{(iii)}}
            + \underbrace{\abs{y_i - \bfx_i^\dagger\theta }}_{\mathrm{(iv)}}.
            \label{eq:gradient_split}
    \end{align}
    We will show that each of the four terms in Equation~\eqref{eq:gradient_split} is at most $\frac \gamma 4\cdot \frac{1}{\multiplier\sqrt{p}}$ with high probability.
    Combined with Condition~\ref{condition:covariate_norm}, this establishes that $\norm{g_{i,t}} = \abs{y_i - \bfx_i^\dagger \theta_t}\cdot \norm{\bfx_i} \le \gamma$, as we desire.
    
    \textbf{Term (i):}
    We decompose the vector $(\theta - \hat\theta)$ \citep[see][Lemma 1]{hsu2011analysis}.
    As in Condition~\ref{condition:noise_covariate_independence}, define matrix $A_t = (\I - (\I-\eta \cov)^t) \cov^{-1}$.
    We have
    \begin{align}
        \abs{\bfx_i^\dagger \bracket{\I - (\I-\eta \cov)^t} (\theta - \hat\theta)}
            &= \abs{\bfx_i^\dagger \bracket{\I - (\I-\eta \cov)^t}  \cdot \frac 1 n \sum_j \cov^{-1} (y_j-\bfx_j^\dagger\theta) \bfx_j} \\
            &= \frac 1 n \abs{ \sum_{j} (y_j-\bfx_j^\dagger\theta) \cdot \bfx_{i}^\dagger A_t \bfx_j} \\
            &\le \frac 1 n \cdot \multiplier \sigma \sqrt{n p},
    \end{align}
    applying Condition~\ref{condition:noise_covariate_independence} in the last line.
    Since $n\ge p$ by assumption, term (i) is at most $\multiplier  \sigma$, which is less than $\frac \gamma 4\cdot \frac{1}{\multiplier \sqrt{p}}$ by assumption.
    
    \textbf{Term (ii):}
    Since $\theta_0 = 0$, Condition~\ref{condition:true_param_direction} directly says that term (ii) is at most $\multiplier$, which is less than $\frac \gamma 4\cdot \frac{1}{\multiplier\sqrt{p}}$ by assumption.
    
    \textbf{Term (iii):}
    Push $\bfx_i^\dagger$ inside the sum and apply the triangle inequality:
    \begin{align}
        \abs{\eta \cdot \bfx_i^\dagger \bignoise_t}
            &= \eta \cdot \abs{\sum_{\ell=1}^t \bfx_i^\dagger (\I - \eta \cov)^{\ell - 1} \noise_{t-\ell}} \\
            &\le \eta \sum_{\ell=1}^t \abs{ \bfx_i^\dagger (\I - \eta \cov)^{\ell - 1} \noise_{t-\ell}}.
    \end{align}
    By Equation~\eqref{eq:noise_bound} (which relies on Conditions~\ref{condition:spectrum} and \ref{condition:covariate_norm}), we have an upper bound on each of these terms that holds with probability at least $1-\beta$.
    Plugging this in, we have
    \begin{align}
        \norm{\eta\cdot  \bfx_i^\dagger \bfz_t'} 
            &\le  \eta \sum_{\ell=1}^t 2\multiplier\lambda \sqrt{p} \bracket{7/8}^{\ell-1} \sqrt{\ln 2nT/\beta} \\
            &\le   2\multiplier  \eta \lambda \sqrt{p}\sqrt{\ln 2nT/\beta}  \sum_{\ell=1}^t\bracket{7/8}^{\ell-1}. 
    \end{align}
    where in the second line we have pulled out the terms that do not depend on $\ell$.
    Because it is a geometric series, the sum is at most $8$.
    Rearranging, we see that term (iii) is at most $\frac \gamma 4\cdot \frac{1}{\multiplier\sqrt{p}}$ when
    \begin{align}
        \frac{\gamma}{\eta \lambda}\ge 64 \multiplier^2 p \sqrt{\ln 2nT/\beta},
    \end{align}
    which is exactly what we assumed.
    
    \textbf{Term (iv):}
    Condition~\ref{condition:response_noise_magnitude} says that term (iv) is at most $\sigma \multiplier$, which is less than $\frac \gamma 4\cdot \frac{1}{\multiplier\sqrt{p}}$ by assumption.
\end{proof}

\begin{lem}\label{lemma:spherically_symmetric}
    Let $\bfx_1,\ldots,\bfx_n$ be drawn i.i.d.~from $\calN(0,\I)$.
    Let $\eta$ be a real number and $t$ an integer.
    Let $\cov = \frac 1 n \sum_i \bfx_i \bfx_i^\dagger$.
    For any $i$, the distribution of $(\I-\eta \cov)^t \bfx_i$ is spherically symmetric.
\end{lem}
\begin{proof}
    Let $\bfv$ be a vector and $\Pi$ an orthogonal rotation matrix. 
    We will show that 
    \begin{align}
        \Pr\sqbracket{(\I-\eta \cov)^t \bfx_i = \bfv} = \Pr\sqbracket{(\I-\eta \cov)^t \bfx_i = \Pi \bfv},
    \end{align}
    using the spherical symmetry of the covariates' distribution.
    We write out
    \begin{align}
        \Pr\sqbracket{(\I-\eta \cov)^t \bfx_i = \bfv}
            &=\Pr\sqbracket{\Pi (\I-\eta \cov)^t \bfx_i = \Pi \bfv} \\
            &= \Pr\sqbracket{\Pi (\I-\eta \cov)\cdots(\I-\eta \cov) \bfx_i = \Pi \bfv} \\
            &= \Pr\sqbracket{\Pi\bracket{\Pi^\dagger \Pi} (\I-\eta \cov)(\Pi^\dagger \Pi)\cdots(\I-\eta \cov) (\Pi^\dagger \Pi) \bfx_i = \Pi \bfv},
    \end{align}
    inserting $\I=\Pi^\dagger \Pi$ between each term.
    We cancel and rearrange:
    \begin{align}
        \Pr\sqbracket{(\I-\eta \cov)^t \bfx_i = \bfv}
            &=\Pr\sqbracket{(\I-\eta (\Pi \cov\Pi^\dagger ))\cdots(\I-\eta (\Pi \cov\Pi^\dagger)) (\Pi \bfx_i) = \Pi \bfv}.
    \end{align}
    Of course, we have $\Pi \cov\Pi^\dagger = \Pi \bracket{\frac 1 n \bfx_i\bfx_i^\dagger}\Pi^\dagger = \frac 1 n \sum_i (\Pi \bfx_i)(\Pi \bfx_i)^\dagger$.
    Therefore, by the rotation invariance of the Gaussian distribution, we arrive at
        $\Pr\sqbracket{(\I-\eta \cov)^t \bfx_i = \bfv}
            = \Pr\sqbracket{(\I-\eta \cov)^t \bfx_i = \Pi \bfv}$.
\end{proof}

\begin{lem}[Restatement of Lemma~\ref{lem:assumption_holds}]
    \assumptionHolds
\end{lem}
\begin{proof}
    the \noclippingcondition has five sub-conditions.
    We prove each holds with probability at least $1-\beta/5$; a union bound finishes the proof.
    To establish Conditions~\ref{condition:response_noise_magnitude}, \ref{condition:true_param_direction}, and \ref{condition:noise_covariate_independence}, we take $\theta\gets\theta^*$.
    
    \textbf{Condition~\ref{condition:spectrum}}
    By Claim~\ref{clm:concentration_covariance}, with probability at least $1-\beta/5$ we have $\frac 1 2 \I \preceq \cov \preceq 2\I$.
    Applying $\eta=1/4$ establishes Condition~\ref{condition:spectrum}.
    
    \textbf{Condition~\ref{condition:covariate_norm}}
    By Claim~\ref{claim:concentration_of_gaussian_norm}, for any fixed $\bfx_i$ we have $\norm{\bfx_{i}} \le \sqrt{p} + \sqrt{2 \ln 5 n/\beta}$ with probability at least $1-\beta/(5n)$.
    Let $c_1 = (1 + \sqrt{2/p}\ln 5n/\beta)$, so we have $\norm{\bfx_i}\le c_1\sqrt{p}$; this is at most $3\sqrt{\ln 5n/\beta}$ and, furthermore, no greater than $\multiplier$.
    A union bound over all $n$ covariates means this condition holds with probability at least $1-\beta/5$.
    
    \textbf{Condition~\ref{condition:response_noise_magnitude}}
    Since $\xi_i \sim \calN(0,\sigma^2)$, we apply Claim~\ref{claim:concentration_univariate_gaussian}: with probability at least $1-\beta/5$, for all $i\in [n]$ we have $\abs{\xi_i}\le \sigma \sqrt{2\ln 10n/\beta}$.
    
    \textbf{Condition~\ref{condition:true_param_direction}}
    Fix $i$ and $t$. 
    Observe that the distribution of $(\I-\eta \cov)^t\bfx_i$ is spherically symmetric and independent of $\theta^*$.
    (Lemma~\ref{lemma:spherically_symmetric}, below, contains a rigorous proof of this statement.)
    Thus we can write it $(\I-\eta \cov)^t \bfx_i = \lambda_{i,t} \bfu_{i,t}$ for $\lambda_{i,t}\in \R$ and $\bfu_{i,t}\in \mathbb{S}^{p-1}$.
    By Lemma~\ref{lem:sub-exp_inner_uniform_dist}, then, with probability at least $1- \beta/(5nT)$ we have 
    \begin{align}
        \abs{\bfx_i^\dagger (\I - \eta \cov)^t \theta^*} &= \lambda_{i,t}\cdot \abs{\bfu_{i,t}^\dagger \theta^*} \\
            &\le \lambda_{i,t} \cdot \frac{4 \ln 5nT/\beta}{\sqrt{p} - 1}.
    \end{align}
    Conditions \ref{condition:spectrum} and \ref{condition:covariate_norm} imply $\lambda_{i,t} = \norm{(\I-\eta)^t\bfx_i}\le \multiplier \sqrt{p}$, so we arrive at $\abs{\bfx_i^\dagger (\I-\eta \cov)^t \theta^*} \le \frac{4 \multiplier \ln 5nT/\beta}{1 -  p^{-1/2}}$, which is at most $16 \multiplier \ln 5nT/\beta$ for $p\ge 2$. 
    A union bound over all $n, T$ implies this holds with probability at least $1-\beta/5$.
    
    \textbf{Condition~\ref{condition:noise_covariate_independence}}
    We want to bound $\abs{\sum_j (y_j - \bfx_j^\dagger\theta^*) \cdot \bfx_i^\dagger A_t \bfx_j}$ for $A_t = (\I-(\I-\eta \cov)^t)\cov^{-1}$.
    Recall that $\xi_j = y_i - \bfx_j^\dagger\theta^*$ by definition.
    For any fixed value of the covariates, the variances sum: we have
    \begin{align}
        \sum_j \xi_j \cdot \bfx_i^\dagger A_t \bfx_j \sim \calN\bracket{0, \sigma^2 \sum_j (\bfx_i^\dagger A_t \bfx_j)^2}.
    \end{align}
    Developing the square, we have
    \begin{align}
        \sum_j (\bfx_i^\dagger A_t \bfx_j)^2 &= \sum_j \bfx_i^\dagger A_t \bfx_j\bfx_j^\dagger A_t^T \bfx_i \\
            &= n\cdot \bfx_i^\dagger A_t \bracket{\frac 1 n \sum_j \bfx_j\bfx_j^\dagger} A_t^T \bfx_i \\
            &= n \cdot \bfx_i^\dagger A_t \cov A_t^T \bfx_i.
    \end{align}
    Plugging in the definition of $A_t$, we cancel and apply Cauchy--Schwarz:
    \begin{align}
        \sum_j (\bfx_i^\dagger A_t \bfx_j)^2 &= n \cdot \bfx_i^\dagger (\I - (\I-\eta \cov)^t) \cov^{-1} \cov \cov^{-1} (\I - (\I-\eta \cov)^t) \bfx_i \\
            &= n \cdot \bfx_i^\dagger (\I - (\I-\eta \cov)^t) \cov^{-1} (\I - (\I-\eta \cov)^t) \bfx_i \\
            &\le n \cdot \norm{(\I - (\I-\eta \cov)^t) \cov^{-1} (\I - (\I-\eta \cov)^t)} \cdot \norm{\bfx_i}^2 \\
            &\le n \cdot \norm{\I - (\I-\eta \cov)^t)}^2 \cdot \norm{\cov^{-1}} \cdot \norm{\bfx_i}^2.
    \end{align}
    Our previous assumptions imply that $\norm{\I - (\I-\eta \cov)^t} \le \norm{\I} + \norm{\I-\eta \cov} \le 2$ and $\norm{\cov^{-1}} \le 2$, so we arrive at 
    \begin{align}
        \sum_j (\bfx_i^\dagger A_t \bfx_j)^2
            &\le 8 c_1^2 n p,
    \end{align}
    plugging in $\norm{\bfx_i}\le c_1 \sqrt{p}$ from Condition~\ref{condition:covariate_norm}.
    
    Thus, with probability at least $1-\beta/5$, for all $i$ and $t$ we have 
    \begin{align}
        \abs{\sum_j \xi_j \cdot \bfx_i^\dagger A_t \bfx_j}
            &\le \sqrt{\sigma^2 \sum_j (\bfx_i^\dagger A_t \bfx_j)^2}\cdot \sqrt{2\ln 5nT/\beta} \\
            &\le \sqrt{\sigma^2 8 c_1^2 n p}\cdot \sqrt{2\ln 5nT/\beta}.
    \end{align}
    Plugging in $c_1 \le 3 \sqrt{\ln 5n/\beta}$ concludes the proof.
\end{proof}

\begin{thm}[Restatement of Theorem~\ref{thm:main_accuracy_claim}, Main Accuracy Claim]
    \mainAccuracyClaim
\end{thm}
\begin{proof}
    By Lemma~\ref{lem:assumption_holds}, with probability at least $1-\beta/4$ we have that data set $(\bfX,\bfy)$ satisfies the \noclippingcondition with $\multiplier=12 \ln^{1.5}(20 nT/\beta)$.
    This uses the assumption that $n\ge c(p+\ln 1/\beta)$.
    
    Under this condition, Lemma~\ref{lem:no_clipping} says that Algorithm~\ref{alg:DPGD} does not clip with probability at least $1-\beta/4$.
    This uses the assumptions that
    \begin{align}
        \gamma &\ge 4\multiplier^2 \sigma \sqrt{p} 
            = 576 \ln^{3}(20 nT/\beta)\times \sigma \sqrt{p}\\ 
        \frac{\gamma}{\eta \lambda} 
            &\ge 64 \multiplier^2 p \sqrt{\ln(2nT/\beta)}.
    \end{align}
    The first is satisfied by construction.
    To see that the second is satisfied, plug in $\eta=\frac 1 4$, $\lambda=\frac{\sqrt{2T} \gamma}{\sqrt{\rho} n}$, and $T=c\log \frac{n\rho}{p}$ to turn this into a lower bound on $n$: omitting constants, it suffices to take
    \begin{align}
        n \gtrsim \sqrt{p} \log^{4}(n\rho/\beta p).
    \end{align}
    This is satisfied when $n\ge c \sqrt{p}\log^4(\rho/\beta)$ for some sufficiently large constant $c$.
    
    This implies, via Lemma~\ref{lem:couple_DPGD_noclipping}, that the total variation distance between the output of Algorithm~\ref{alg:DPGD} and the same algorithm without clipping (i.e., $\gamma=+\infty$) is at most $\beta/2$.
    Thus, if we prove an error bound for Algorithm~\ref{alg:DPGD} with $\gamma=\infty$ that holds with probability at least $1-\beta/2$, then the same guarantee holds for the clipped version of Algorithm~\ref{alg:DPGD} with probability at least $1-\beta$.
    
    Lemma~\ref{lem:distribution_of_DPGD_no_clipping} gives us the explicit distribution for the algorithm's final iterate $\theta_T$.
    We take the $\ell_2$ norm and apply the triangle inequality:
    \begin{align}
        \norm{\theta_T -\hat\theta}_2 \le \norm{(\I-\eta\Sigma)^T\hat\theta}_2 + \norm{\eta \cdot \bfz_T'}_2. \label{eq:acc_bias_variance}
    \end{align}
    Here $\bfz_T'$ is a noise term: 
    defining $\bfD = (\I-\eta\cov)^2$ as in Lemma~\ref{lem:distribution_of_DPGD_no_clipping}, we have $\noise_T' \sim \calN(0,\lambda^2 \bfA^{(T)})$ for 
    \begin{align}
        \bfA = (\I-\bfD)^{-1}(\I-\bfD^T).
    \end{align}
    The \noclippingcondition gives us upper and lower bounds on $\bfD$: in particular, we have $\norm{\bfA}\le 16$.
    Thus, with probability at least $1-\beta/4$, we have $\norm{\eta \cdot \noise_T'}\le \eta \lambda \sqrt{32 p \ln 4/\beta}$.
    
    The second term in Equation~\eqref{eq:acc_bias_variance} decays exponentially with $T$.
    By Cauchy--Schwarz, it is at most $\norm{(\I-\eta\Sigma)^T}\cdot \norm{\hat\theta} \le (7/8)^T \norm{\hat\theta}$, applying Condition~\ref{condition:main} to bound the operator norm.
    We bound the norm of $\hat\theta$ based on its distance to $\theta^*$:
    \begin{align}
        \norm{\hat\theta} = \norm{\hat\theta - \theta^* + \theta^*}
            &\le \norm{\theta^*} + \norm{\hat\theta - \theta^*} \\
            &= \norm{\theta^*} + \norm{\frac 1 n \sum_j \cov^{-1} \xi_j \bfx_j}.
    \end{align}
    The first norm is at most 1 by assumption; a rough bound suffices for the second:
    \begin{align}
        \norm{\frac 1 n \sum_j \cov^{-1} \xi_j \bfx_j}
            &\le \frac 1 n \sum_j \norm{\cov^{-1}} \cdot \abs{y_j-\bfx_i^\dagger \theta^*}\cdot \norm{\bfx_j} 
            \le 2 \multiplier^2 \sigma \sqrt{p},
    \end{align}
    applying our assumptions. 
    
    Plugging these bounds back into Equation~\eqref{eq:acc_bias_variance} and plugging in our expressions for $\eta$, $\lambda$, and $\gamma$, we have (omitting constants)
    \begin{align}
    	\norm{\theta_T - \hat\theta} 
    		&\lesssim \sigma \sqrt{p} \log^{3}\bracket{\frac{nT}{\beta}}\bracket{\frac 7 8}^T
    			+ \eta \lambda \sqrt{p \log 1/\beta} \\
    		&\lesssim \sigma \sqrt{p} \log^{3}\bracket{\frac{nT}{\beta}}\bracket{\frac 7 8}^T
    			+ \frac{\sigma p \sqrt{T}}{\sqrt{\rho} n}\cdot \log^{3.5}(nT/\beta).
    \end{align}
    The first term is dominated by the second when $T\ge c \log (n\rho/p)$. 
    Substituting in this value of $T$ finishes the proof.
\end{proof}

\subsection{Confidence Intervals (Proof of Theorem~\ref{thm:CI_coverage})}

\subsubsection{Confidence Intervals for Independent Draws}

A fundamental task in statistical inference is to produce confidence intervals for the mean given independent samples from a normal distribution with unknown mean and variance.
Claim~\ref{clm:CI_ideal_case} reproduces this basic fact.
After that, we observe that confidence intervals produced in this way are valid for samples from distributions that are close in TV distance.
\begin{claim}
\label{clm:CI_ideal_case}
Let $\{\theta_i\}_{i\in[m]}$ be $m$ samples drawn i.i.d. from $\calN(\mu,\Sigma)$, for any $\mu\in\R^p$ and $\Sigma\in\R^{p\times p}$.
For any $j\in[p]$, let $\bar\theta_j = \frac 1 m \sum_{i=1}^m (\theta_i)_j$ be the sample mean
and 
$\hat\sigma_j^2 = \frac{1}{m-1}\sum_{i=1}^m \bracket{ \bracket{\theta_i}_j - \bar\theta_j }^2$ the sample variance.
Then $(\bar\theta)_j \pm t_{\alpha/2,m-1}\cdot \frac{\hat\sigma_j}{\sqrt m}$,
is a $1-\alpha$ confidence interval.
Here $t_{\alpha,m-1}$ denotes the $\alpha$ percentile for the student's $t$ distribution with $m-1$ degrees of freedom.
\end{claim}

\begin{claim}\label{clm:CIs_close_in_tv}
    Suppose mechanism $\calM:\R^{m}\to \R^2$, when given samples $z_1,\ldots,z_m$ drawn i.i.d.\ from a distribution $\calN(\mu,\sigma^2)$, produces a $1-\alpha$ confidence interval for $\mu$.
    Let $q$ be a distribution over $\R^m$.
    If $(z_i',\ldots,z_m')\sim q$, then $\calM(z_1',\ldots,z_m')$ produces a $1-\alpha - \mathrm{TV}(q, \calN(\mu,\sigma^2)^{\otimes m})$ confidence interval for $\mu$.
\end{claim}

\subsubsection{Analyzing Independent Runs and Checkpoints}

Lemma~\ref{lem:distribution_of_DPGD_no_clipping} tells us that the $t$-th iterate of DP-GD without clipping has the distribution
\begin{align}
    \theta_t = \hat{\theta} + (\I-\eta\Sigma)^t(\theta_0-\hat\theta) + \bfz_t',
\end{align}
where $\bfz_t'\sim \calN(0, \eta^2\lambda^2 \bfA^{(t)})$ for 
$\bfA^{(t)}= (\I-\bfD)^{-1}(\I-\bfD^t)$ and $\bfD = (\I - \eta \Sigma)^2$.
Recall $\eta$ the step size and $\Sigma=\frac 1 n \bfX^\dagger \bfX$ the empirical covariance.
For appropriate $\bfD$, as $t$ grows this approaches the distribution $\calN(\hat\theta, \eta^2 \lambda^2 \bfA^{(\infty)})$, where $\bfA^{(\infty)}=(\I - \bfD)^{-1}$.

In this section, we give guarantees for two algorithms that constructing confidence intervals: independent runs and checkpoints. 
We show that each produces a sequence of $m$ vectors which, as a whole, is close in total variation distance to $m$ independent draws from $\calN(\hat\theta, \eta^2 \lambda^2 \bfA^{(\infty)})$.
We call this latter the ``ideal'' case, as it corresponds to Claim~\ref{clm:CI_ideal_case}.
The other two cases are ``independent runs'' and ``checkpoints.''
Thus, we consider three sets of random variables: 
$\theta_1^{(1)},\ldots,\theta_m^{(1)}$ are drawn i.i.d.\ from $\calN(\hat\theta, \eta^2 \lambda^2 \bfA^{(\infty)})$;
$\theta_1^{(2)},\ldots,\theta_m^{(2)}$, where $\theta_\ell^{(2)}$ is $t$-th iterate of an independent run of DP-GD without clipping; and
$\theta_1^{(3)},\ldots,\theta_m^{(3)}$, where $\theta_\ell^{(2)}$ is the $\ell t$-th iterate of a single run of DP-GD without clipping.
We interpret each of these as a draw from a multivariate Gaussian distribution in $mp$ dimensions: we define the concatenated vector $\theta^{(1)} = \sqbracket{\theta_1^{(1)} \bigl|  \cdots\bigl|\theta_m^{(1)}}$ and let $\mu^{(1)}\in \R^{mp}$ and $\Sigma^{(1)}\in \R^{mp\times mp}$ be its mean and covariance, respectively.
We define the analogous notation for the other two collections.
Our calculations will only require us to index into these objects ``blockwise,'' so we abuse notation and define, for $\ell,k\in[m]$,
\begin{align}
    \mu^{(1)}_\ell &= \E\sqbracket{\theta_\ell^{(1)} }\in\R^p  \quad\text{and}\quad
    \Sigma_{\ell,k}^{(1)} = \E\sqbracket{\bracket{\theta_\ell^{(1)}-\mu_\ell^{(1)}}\bracket{\theta_k^{(1)} -\mu_k^{(1)}}^\dagger}\in \R^{p\times p}.
\end{align}
We establish that the vectors from our algorithms are close in total variation to the ideal case by showing that the relevant means and covariances are close. 
We use the following standard fact.
\begin{claim}[See, e.g., \cite{diakonikolas2019robust}]\label{claim:gaussian_tv_distance}
    There exists a constant $K$ such that, for any $\beta\le \frac 1 2$, vectors $\mu_1,\mu_2\in\R^d$, and positive definite $\Sigma_1,\Sigma_2\in\R^{d\times d}$, if $\norm{\mu_1-\mu_2}_{\Sigma_1}\le \beta$ and $\norm{\Sigma_1^{-1/2}\Sigma_2\Sigma_1^{-1/2}-\I}_F\le\beta$ then $\mathrm{TV}(\calN(\mu_1,\Sigma_1),\calN(\mu_2,\Sigma_2)) \le K\beta$.    
\end{claim}

\paragraph{Ideal Case}
Each vector is an independent draw from a fixed Gaussian, so for any $\ell$ and any $k\neq \ell$ we have
\begin{align}
    \mu_\ell^{(1)} &= \hat\theta, \quad
    \Sigma_{\ell,\ell}^{(1)} = \eta^2\lambda^2\bfA^{(\infty)}, \text{ and} \quad
    \Sigma_{\ell,k}^{(1)} = 0.
\end{align}

\paragraph{Independent Runs}
Here the vectors are also independent Gaussian random variables, but with different parameters.
For any $\ell$ and any $k\neq \ell$ we have
\begin{align}
    \mu_\ell^{(2)} = \hat\theta + (\I-\eta\Sigma)^t (\theta_0 - \hat\theta),\quad
    \Sigma_{\ell,\ell}^{(2)} &= \eta^2\lambda^2 \bfA^{(t)},\text{ and}  \quad
    \Sigma_{\ell,k}^{(2)} = 0.
\end{align}

\paragraph{Checkpoints}
Here the vectors are no longer independent, so we have additional work.
We have
\begin{align}
    \mu_\ell^{(3)} = \hat\theta + (\I-\eta \Sigma)^{\ell t} (\theta_0 - \hat\theta),\quad
    \Sigma_{\ell,\ell}^{(3)} &= \eta^2\lambda^2  \bfA^{(\ell t)}, \text{ and}  \quad
    \Sigma_{\ell,k}^{(3)} = \E\sqbracket{(\bfz_{\ell t}') (\bfz_{k t}')^\dagger}.
\end{align}
We now analyze this third term.
Recall from Lemma~\ref{lem:expanded_dist_no_clipping} that these vectors $\bfz_{\ell t}'$ stand in for sums that depend on all the noise vectors added so far:
\begin{align}
    \bfz_t' = \eta\sum_{i=1}^t (\I-\eta\Sigma)^{i-1} \bfz^{t-i} = \eta\sum_{i=0}^{t-1} (\I-\eta\Sigma)^{t-i-1} \bfz^{i},
\end{align}
where the second equation was re-indexed to simplify the next operation.
The random variables $\bfz_i$ are drawn i.i.d.\ from $\calN(0,\lambda^2 \I)$.
This allows us to understand the covariance in $\Sigma_{\ell,k}^{(3)}$, as some of the noise terms are duplicated and some are not.
For any pair of integers $\tau>t$, we have
\begin{align}
    \bfz_\tau' &= \eta \sum_{i=0}^{\tau-1} (\I-\eta \Sigma)^{\tau-i-1}\bfz^i \\
        &=\eta \sum_{i=0}^{t-1} (\I-\eta \Sigma)^{\tau-i-1}\bfz^i + \eta \sum_{i=t}^{\tau-1} (\I-\eta \Sigma)^{\tau-i-1}\bfz^i \\
        &=(\I-\eta\Sigma)^{\tau-t} \cdot \eta \sum_{i=0}^{t-1} (\I-\eta \Sigma)^{t-i-1}\bfz^i + \eta \sum_{i=t}^{\tau-1} (\I-\eta \Sigma)^{\tau-i-1}\bfz^i \\
        &= (\I-\eta\Sigma)^{\tau-t} \bfz_t' + \eta \sum_{i=t}^{\tau-1} (\I-\eta \Sigma)^{\tau-i-1}\bfz^i.
\end{align}
The second sum is independent of $\bfz_t'$, since it involves only later terms in the algorithm.
Applying this with $\tau \gets \ell t$ and $t \gets kt$ (assuming without loss of generality that $\ell>k$), we have
\begin{align}
    \Sigma_{\ell,k}^{(3)} = \E\sqbracket{(\bfz_{\ell t}') (\bfz_{kt}')^\dagger}
        &= (\I-\eta\Sigma)^{(\ell-k)t} \E\sqbracket{(\bfz_{k t}') (\bfz_{kt}')^\dagger} \\
        &=\eta^2\lambda^2 (\I-\eta\Sigma)^{(\ell-k)t} \bfA^{(kt)}.
\end{align}

\paragraph{Setup}
We derive some facts which will be useful for both of our case-specific analyses.
Since the data satisfies the \noclippingcondition, we have $\frac 1 2 \I \preceq  \Sigma \preceq 2\I$ and $\norm{\I-\eta \Sigma}\le \frac 7 8$.
These in turn establish operator norm bounds for $\bfD, \bfA^{(\infty)}, \I-\bfD$ (and their inverses) of at most $c$ for some absolute constant $c$, which yields a Frobenius norm bound of $c\sqrt{m}$.
We also have $\norm{\bfD^t}_F\le \sqrt{m}\cdot \bracket{\frac{7}{8}}^t$.

We pass from the rescaled norms needed in Claim~\ref{claim:gaussian_tv_distance} to Frobenius and $\ell_2$ norms:
\begin{align}
    \norm{(\Sigma^{(1)})^{-1/2}(\Sigma^{(2)})(\Sigma^{(1)})^{-1/2} - \I}_F
        &\le \frac{1}{\lambda_{\min}(\Sigma^{(1)})} \cdot \norm{\Sigma^{(2)}- \Sigma^{(1)}}_F, \quad\text{and} \label{eq:covariance_breakdown}\\
    \norm{\mu^{(1)}-\mu^{(2)}}_{\Sigma^{(1)}} &\le \frac{1}{\sqrt{\lambda_{\min}(\Sigma^{(1)})}}\cdot \norm{\mu^{(1)}-\mu^{(2)}}_2. \label{eq:mean_breakdown}
\end{align}
Since the minimum eigenvalue of a block-diagonal matrix is the minimum eigenvalue of the blocks and, in the ideal case, the blocks are diagonal, we have $\lambda_{\min}(\Sigma^{(1)}) = \lambda_{\min}(\eta^2\lambda^2 \bfA^{(\infty)}) \ge \eta^2\lambda^2 c$ for some constant $c$.

\paragraph{Analyze Independent Runs}
With the above tools in hand, it suffices to analyze $\norm{\Sigma^{(2)}- \Sigma^{(1)}}_F$.
Observe that $\Sigma^{(2)}$ is equal to $\Sigma^{(1)}$ plus a block-diagonal matrix where each block is $\eta^2 \lambda^2$ times $\bfA^{(\infty)}-\bfA^{(t)}=  \bfA^{(\infty)}\bfD^t$.
So we have
\begin{align}
    \norm{\Sigma^{(2)}- \Sigma^{(1)}}_F
        \le\eta^2\lambda^2 m\cdot \norm{\bfA^{(\infty)}\bfD^t}_F 
        &\le c \eta^2 \lambda^2 m^{2} \bracket{\frac 7 8}^t
\end{align}
for some constant $c$.
Cancelling the $\eta^2\lambda^2$ from $\lambda_{\min}$ in Equation~\eqref{eq:covariance_breakdown}, we have
    $\norm{(\Sigma^{(1)})^{-1/2}(\Sigma^{(2)})(\Sigma^{(1)})^{-1/2} - \I}_F \le \beta$
when $t\ge c \log m/\beta$ for some constant $c$.

The mean is similar: the $m$ blocks of the vector are identical, so we have
\begin{align}
    \norm{\mu^{(1)} - \mu^{(2)}}_2 \le \sqrt{m} \norm{(\I-\eta\Sigma)^t (\theta_0-\hat\theta)}_2.
\end{align}
The \noclippingcondition implies that $\norm{\theta_0-\hat\theta}_2\le 2c_0^2 \sigma\sqrt{p}$ (as in the proof of Theorem~\ref{thm:main_accuracy_claim}).
Thus, combining with Equation~\eqref{eq:mean_breakdown}, we have
    $\norm{\mu^{(1)} - \mu^{(2)}}_{\Sigma^{(1)}} \le \beta$
when $t \ge c \log \frac{\sigma m p}{\eta \lambda \beta}$ for some constant $c$.

\paragraph{Analyze Checkpoints}
We apply similar arguments.
The upper bound on $\norm{\mu^{(1)} - \mu^{(3)}}_{\Sigma^{(1)}}$ is identical to the previous case except that we apply $\norm{\bfD^{\ell t}}_F\le \norm{\bfD^{t}}_F$ for all $\ell\in [m]$.
The covariance analysis is similar, but $\Sigma^{(1)}-\Sigma^{(3)}$ is no longer block-diagonal (as the checkpoints are not independent). 
So the main additional work is to control the effect of these off-diagonal blocks.

We expand over the $m^2$ blocks, moving temporarily to the squared Frobenius norm:
\begin{align}
    \norm{\Sigma^{(3)} - \Sigma^{(1)}}_F^2
        &= \eta^4\lambda^4\bracket{\sum_{\ell=1}^m \norm{\bfA^{(\infty)} - \bfA^{(\ell t)}}_F^2 + 2\sum_{\ell > k} \norm{0 - (\I-\eta\Sigma)^{(\ell-k)t} \bfA^{(kt)}}_F^2} \\
        &\le \eta^4\lambda^4\bracket{\sum_{\ell=1}^m \norm{\bfA^{(\infty)}\bfD^t}_F^2 + 2\sum_{\ell > k} \norm{(\I-\eta\Sigma)^{(\ell-k)t}}_2^2 \cdot \norm{\bfA^{(kt)}}_F^2}.
\end{align}
As before, we combine with Equation~\eqref{eq:covariance_breakdown} to obtain an upper bound on the (unsquared) norm of $\poly{m}\cdot \bracket{\frac 7 8}^t$, which is less than $\beta$ when $t\ge c \log m/\beta$ for some constant $c$.

\end{document}